\documentclass[11pt]{article}

\usepackage[margin=1in]{geometry}
\usepackage{smile}

\usepackage{amscd,amssymb,stmaryrd}
\usepackage{enumitem}
\usepackage{amsmath}
\usepackage{graphicx}
\usepackage[utf8]{inputenc}
\usepackage[export]{adjustbox}
\usepackage{wrapfig}
\usepackage{amstext}
\usepackage{pstricks,pst-node,pst-plot,pst-coil}
\usepackage{amsthm}

\usepackage{amsmath}
\usepackage{color}
\usepackage{mathtools}

\usepackage[colorlinks, linkcolor=blue, anchorcolor=blue, citecolor=blue,hypertexnames=false]{hyperref}
\usepackage[english]{babel}
\usepackage{natbib}
\usepackage{booktabs}
\usepackage{mathrsfs}
\usepackage{comment}
\usepackage{float}

\newcommand{\ReLU}{\mathrm{ReLU}}
\newcommand{\supp}{\mathrm{supp}}
\newcommand{\CL}{\mathrm{CL}}
\newtheorem{setting}{Setting}
\usepackage{tikz}
\usepackage{pifont}% 
\usepackage{lineno}
\usepackage{diagbox}
\usetikzlibrary{decorations.pathreplacing,calligraphy}
\usepackage{pgfplots}
\pgfplotsset{compat=1.11}
\usepgfplotslibrary{groupplots}

\allowdisplaybreaks

\providecommand{\keywords}[1]{\textbf{Key words:} #1}

\title{ Neural Scaling Laws of Deep ReLU and Deep Operator Network: \\A Theoretical Study}

\date{}

\author{Hao Liu \thanks{Department of Mathematics, Hong Kong Baptist University, Hong Kong, China. \texttt{Email: haoliu@hkbu.edu.hk.} Supported by National Natural Science Foundation of China  12201530, HKRGC ECS 22302123.}
\and
Zecheng Zhang \thanks{Corresponding Author.  Department of ACMS, University of Notre Dame, IN 46556.  \texttt{Email: zecheng.zhang.math@gmail.com.} Supported by DOE DE-SC0025440.}
\and
Wenjing Liao \thanks{School of Mathematics, Georgia Institute of Technology, Atlanta, GA 30332. \texttt{Email: wliao60@gatech.edu.} Supported by NSF DMS 2145167 and DOE SC0024348.}
\and
Hayden Schaeffer \thanks{Department of Mathematics, UCLA, Los Angeles, CA 90095. \texttt{Email: hayden@math.ucla.edu.} Supported in part by NSF 2427558 and NSF 2331033. }
}

\begin{document}

\maketitle
\begin{abstract}

Neural scaling laws play a pivotal role in the performance of deep neural networks and have been observed in a wide range of tasks. However, a complete theoretical framework for understanding these scaling laws remains underdeveloped. 
In this paper, we explore the neural scaling laws for deep operator networks, which involve learning mappings between function spaces, with a focus on the Chen and Chen style architecture. These approaches, which include the popular Deep Operator Network (DeepONet), approximate the output functions using a linear combination of learnable basis functions and coefficients that depend on the input functions. 
We establish a theoretical framework to quantify the neural scaling laws by analyzing its approximation and generalization errors. 
We articulate the relationship between the approximation and generalization errors of deep operator networks and key factors such as network model size and training data size. 
Moreover, we address practical cases where input functions exhibit low-dimensional structures, allowing us to derive tighter error bounds. These results also hold for deep ReLU networks and other similar network structures. Our results offer a partial explanation of the neural scaling laws in operator learning and provide a theoretical foundation for their applications.
\end{abstract}

\keywords{deep operator learning,  neural scaling law, approximation theory, generalization theory }

\section{Introduction}

Deep neural networks have demonstrated remarkable performance in a wide range of applications, such as computer vision \citep{he2016deep,creswell2018generative}, natural language processing \citep{graves2013speech}, speech recognition \citep{hinton2012deep}, scientific computing \citep{han2018solving,khoo2021solving,zhang2023belnet,bao2024wanco}, etc. In many of these applications, the core problem is to learn an operator between function spaces. 
For example, in \citet{bhattacharya2021model,lifourier},  deep neural networks are used  to represent a solution map of Partial Differential Equations (PDEs), in which the network maps the initial/boundary conditions to PDE solutions. In \citet{ronneberger2015u}, deep neural networks are used for image segmentation, in which the network represents an operator from any given image to its segmented counterpart.

In literature, many network architectures have been proposed to learn operators between function spaces, such as Chen and Chen neural operators \citep{chen1995universal, chen1993approximations}, Principal Component Analysis-based Networks (PCANet) \citep{bhattacharya2021model}, Fourier Neural Operator (FNO) \citep{lifourier}, Deep Operator Network (DeepONet) \citep{lu2021learning}, Autoencoder-based Networks (AENet) \citep{kontolati2023learning,liu2024generalization} and Basis Enhanced Learning (BelNet) \citep{zhang2023belnet}. Since directly learning an operator is difficult due to the curse of dimensionality, a popular strategy is to use an encoder-decoder framework, i.e., one encodes infinite-dimensional functions into finite-dimensional latent features and then learns a map in the latent space. 
FNO \citep{lifourier} uses the Fourier transform to convert computation to the frequency domain and then the map is learned in the frequency domain. PCANet \citep{bhattacharya2021model} uses Principal Component Analysis (PCA) for encoding and decoding. 
DeepONet \citep{lu2021learning, lin2023b} uses a branch net to convert input functions to a set of coefficients, and a trunk net to learn a set of basis functions in the output space. The resulting neural operator in DeepONet is a linear combination of the bases weighted by the coefficients. A novel training strategy of DeepONet is recently proposed in \citet{lee2024training}.

In the training of neural networks, neural scaling laws are observed in regard to the scaling between the generalization error and the data size/model size/running time \citep{kaplan2020scaling}. 
Neural scaling laws between the generalization error and the data size/model size are also empirically observed for operator learning \citep{lu2021learning,de2022cost,li2020neural, subramanian2024towards}. 
For example, \citet{lu2021learning} reported an exponential convergence of the DeepONet test error as the training data size increases for small training datasets, and a polynomial convergence for moderate and large training datasets. \citet{de2022cost} reported a power law between the test error and the training data size on various examples of learning PDE solutions. 
In the multi-operator learning foundation model for PDE \citep{liu2023prose, sun2024lemon}, the authors observed a heuristic scaling law of the testing error as the number of distinct families of operators increase \citep{sun2024lemon}; similar results were noted when scaling up the dataset diversity in climate models \citep{bodnar2024aurora}, where the authors additionally reported a power scaling law with increasing model size. The difficulty with PDE foundation scaling laws is that they dependent on increasing the dataset heterogeneity, since the data sequences cannot be i.i.d. due to temporal dependencies \citep{liu2024prose}.

Neural scaling laws are often used to quantify the performance of neural networks with respect to the data size/model size/running time.  A theoretical understanding of neural scaling laws is of crucial importance, which allows one to analyze and quantify the generalization error in deep learning,  and predicts how much the network performance can be improved by increasing the data size, model size, and running time \citep{hestness2017deep,kaplan2020scaling}.
A theoretical understanding of model/data scaling laws (scaling between the generalization error and model/data size)  can be related to neural network representation and generalization theory. 
When feedforward ReLU networks are used for function approximation, the representation theory in \citet{yarotsky2017error,lu2021deep} quantifies the network approximation error with respect to the model size, which partially explains model scaling laws. 
Data scaling laws can be justified through the generalization error bound in terms of the data size.
It was shown that when feedforward neural networks \citep{schmidt2020nonparametric} and convolutional neural networks \citep{oono2019approximation,yang2024rates} are used for the regression of $s$-H\"older functions in $\RR^D$, the squared generalization error converges on the order of $n^{-\frac{2s}{2s+D}}$ where $n$ denotes the training data size. Similar error bounds are also established for piecewise smooth functions in \citet{petersen2018optimal,imaizumi2019deep,liu2024deepadaptive}
Due to the curse of dimensionality, this rate converges slowly when the data dimension is high ($D$ is large). One way to mitigate the curse of data dimension and improve the rate is by incorporating low-dimensional data structures  \citep{tenenbaum2000global,pope2021intrinsic}. Under a manifold hypothesis, one can achieve the same approximation error with a much smaller network size \citep{chen2019efficient,liu2021besov}, and the squared generalization error is improved to the order of $n^{-\frac{2s}{2s+d}}$ where $d$ is the intrinsic dimension of data \citep{nakada2020adaptive,dahal2022deep,chen2022nonparametric,liu2024deepauto}.

Compared to regression, theoretical analysis of neural scaling laws for operator learning is less studied. An approximation result for PCANet was established in \citet{bhattacharya2021model} . A thorough study on the approximation error of PCANet was conducted in \citet{lanthaler2023operator}, which derived both the upper and lower complexity bounds. The generalization error of an encoder-decoder framework for operator learning was studied in \citet{liu2024deep}. This encoder-decoder framework assumes that the encoders and decoders are either given or estimated from data, and a network is used to learn the mapping between latent spaces. This encoder-decoder framework includes PCANet as a special case. The generalization error derived in \citet{liu2024deep} consists of a network estimation error and an encoding error. The squared network estimation error for Lipschitz operators is on the order of $n^{-\frac{2}{2+d_U}}$ where $d_U$ is the dimension of the input latent space. Furthermore, if the input functions exhibit a single-chart manifold structure with intrinsic dimension $d_U$ and the latent variables are learned by Autoencoder, \citet{liu2024generalization} provided a generalization error analysis where the squared generalization error is on the order of $n^{-\frac{1}{2+d_U}}$. 

Regarding \citet{chen1995universal, chen1993approximations} style neural operators such as the popular DeepONet \citep{lu2021deep}, the first universal approximation theory was established in \citet{chen1995universal, chen1993approximations}. 
The authors showed that DeepONet \citep{lu2021deep, lu2021learning} can approximate continuous operators with arbitrary accuracy, the authors in \cite{zhang2023belnet,zhang2023discretization} later extended the theorem to be invariant to the discretization. 
However, the network size was not specified in \citet{chen1995universal, lu2021deep} and therefore this theory cannot quantify model scaling laws.
 A more comprehensive analysis of DeepONet was conducted in \citet{lanthaler2022error}, which studied the approximation error of each component in DeepONet with an estimation on the network size. 
  These results were applied to study several concrete problems on the solution operator of differential equations. A generalization error was also studied in \citet{lanthaler2022error}, which focused on the stochastic error (variance). The bias-variance trade-off was not addressed and the neural scaling law is not explicitly provided. 

In this paper, we study the neural scaling laws of Chen-Chen style neural operators. Specifically, let $U$ and $V$ be two functions sets with domain dimensions $d_1$ and $d_2$ respectively, and $G: U\rightarrow V$ be a Lipschitz operator between $U$ and $V$. We consider learning Lipschitz operators by DeepONet and analyze its approximation error and generalization error.
Our main results are summarized as follows and in Table \ref{tablesummary}:
\begin{enumerate}
    \item We show that if the network architecture is properly set, DeepONet can approximate Lipschitz operators with arbitrary accuracy. In particular, if we denote the number of network parameters by $N_{\#}$, the approximation error of DeepONet for Lipschitz operators is on the order of $\left(\frac{\log N_{\#}}{\log\log N_{\#}}\right)^{-\frac{1}{d_1}}$.
    
    \item We prove that the squared generalization error of DeepONet for learning Lipschitz operators is on the order of {\color{black}$\left(\frac{\log n}{\log \log n}\right)^{-\frac{2}{d_1}}$}, where $n$ is the number of input-output function pairs in the training data.
    
    \item Furthermore, we incorporate low-dimensional structures of input functions into our analysis and improve the power law in $\log N_{\#}$ and $\log n$ above to a power law in $N_{\#}$ and $n$ respectively.
    Specifically, when all functions in $U$ can be represented by $b_U$ orthogonal bases,  
    the approximation error of DeepONet for Lipschitz operators is on the order of $N_{\#}^{-\frac{1}{b_U+d_2}}$, and the squared generalization error is on the order of {\color{black}$n^{-\frac{2}{2+b_U+d_2}}$} up to some logarithmic factor.
\end{enumerate}

\begin{table}[t!]
\centering
\begin{tabular}{ |c |c |c| }
\hline
  & Approximation Error & Squared Generalization Error \\
  \hline
 General Case & $\left(\frac{\log N_{\#}}{\log\log N_{\#}}\right)^{-\frac{1}{d_1}}$ & {\color{black}$\left(\frac{\log n}{\log \log n}\right)^{-\frac{2}{d_1}}$} \\
 \hline
 $U$ Expanded by $b_U$ Bases & $N_{\#}^{-\frac{1}{b_U+d_2}}$ & {\color{black}$n^{-\frac{2}{2+b_U +d_2}}$}   
 \\
 \hline
\end{tabular}
\caption{Summary of the orders of our approximation and generalization error bounds of DeepONet for Lipschitz operators. $N_{\#}$ denotes the network model size, $n$ is the number of input-output function pairs in the training data. $U$ is the input set. $d_1$ and $d_2$ are the dimension of input domain $\Omega_U$ and output domain $\Omega_V$, respectively.   }
\label{tablesummary}
\end{table}

Our results establish novel approximation and generalization error bounds of a class of neural operators originated from \cite{chen1995universal, lu2021deep}, which provide a theoretical justification of neural scaling laws. The slow convergence rate given by the power law in $\log N_{\#}$ and {\color{black}$\log n$} in the general case demonstrates the difficulty of learning general Lipschitz operators without additional data structures. This difficulty is also discussed in \citet{mhaskar1997neural,lanthaler2023curse}. By utilizing  low-dimensional data structures, the neural scaling law is significantly improved to a power law in $N_{\#}$ (model size) and {\color{black}$n$} (data size),
which partially explains the observed power scaling laws in many existing works \citep{de2022cost,lu2021learning}. 

This paper is organized as follows: We introduce related concepts and notations in Section \ref{sec.preliminary}. 
The problem setup and DeepONet structure are presented in Section \ref{sec.setup}. We present our main results in Section \ref{sec.mainresults}: Section \ref{sec.approximation} for the approximation theory and \ref{sec.generalization}  for the generalization theory of learning general Lipschitz operators, and Section \ref{sec.lowd} for an error analysis incorporating low-dimensional data structures. Our main results are proved in Section \ref{sec.main_proof}. We conclude this paper in Section \ref{sec.conclusion}. All proofs of auxiliary lemmata and theorems are deferred to the appendix. 

\section{Preliminary}
\label{sec.preliminary}
\subsection{Neural Network }
In this paper, we define a feedforward ReLU network $ {q}: \mathbb{R}^{d_1}\rightarrow \mathbb{R} $ as
\begin{align}
	q(\xb)=W_L\cdot\ReLU\left( W_{L-1}\cdots \ReLU(W_1 x+b_1)+ \cdots +b_{L-1}\right)+b_L,
	\label{eqn_relu_net}
\end{align}
where $W_l$'s are weight matrices, $b_l$'s are bias vectors, $\ReLU(a)=\max\{a,0\}$ is the rectified linear unit activation (ReLU) applied element-wise, and $\Omega$ is the domain.
We define the network class $\mathcal{F}_{\rm NN}: \mathbb{R}^{d_1} \rightarrow \mathbb{R}^{d_2}:$
\begin{align}
	\cF_{\rm NN}(d_1, d_2, L, p,K, \kappa, R)=\{&[q_1, q_2,...,q_{d_2}]^{\intercal}\in\mathbb{R}^{d_2}: \mbox{ for each }k=1,...,d_2,\nonumber\\
	&q_k:\RR^{d_1}\rightarrow \mathbb{R} \mbox{ is in the form of (\ref{eqn_relu_net}) with } L \mbox{ layers, width bounded by } p, \nonumber\\
	& \|q_k\|_{L^{\infty}}\leq R, \ \|W_l\|_{\infty,\infty}\leq \kappa, \ \|b_l\|_{\infty}\leq \kappa,\  \sum_{l=1}^L \|W_l\|_0+\|b_l\|_0\leq K, \ \forall l   \},
	\label{eq.FNN}
\end{align}
where
$
\|q\|_{L^{\infty}(\Omega)}=\sup\limits_{\xb\in \Omega} |q(\xb)|,\ \|W_l\|_{\infty,\infty}=\max\limits_{i,j} |W_{i,j}|,\ \|b\|_{\infty}=\max\limits_{i} |b_i|
$,
and $\|\cdot\|_0$ denotes the number of nonzero elements of its argument. The network class above has input dimension $d_1$, output dimension $d_2$, $L$ layers, width $p$, the number of nonzero parameters no larger than $K$. All parameters are bounded by $\kappa$ and each element in the output is bounded by $R$.

\subsection{Cover and Partition of Unity}
We define the cover of a set as follows:
\begin{definition}[Cover]
    A collection of sets $\{S_{k}\}_{k=1}^{C_S}$ is a cover of $\Omega$ if $\Omega\subset \bigcup_{k=1}^{C_S}S_k$.
\end{definition}

The following lemma shows that, for a compact smooth manifold $\cM$ and any given cover of $\cM$, 
there exists a $C^{\infty}$ partition of unity of  $\cM$ that subordinates to the given cover.
\begin{lemma}[Theorem 13.7(ii) of \cite{tu2011manifolds}]\label{lemma_pou}
        Let $\{\Omega_k\}_{k=1}^M$ be an open cover of a compact smooth manifold $\cM$ . There exists a $C^{\infty}$ partition of unity $\{\omega_k\}_{k=1}^M$ that subordinates to $\{\Omega_k\}_{k=1}^M$ such that $\supp(\omega_k)\subset \Omega_k$ for any $k$.
    \end{lemma}

    \subsection{Lipschitz Functional}
A Lipschitz functional is defined as follows:
\begin{definition}[Lipschitz functional]
\label{def.lip_functional}
    Given a function set $U$ with domain $\Omega_U$ such that $U\subset L^2(\Omega_U)$, we say a functional $f: U \rightarrow \RR$ is Lipschitz with Lipschitz constant $L_f$ if 
    $$
    |f(u_1)-f(u_2)|\leq L_f\|u_1-u_2\|_{L^2(\Omega_U)}, \ \forall u_1,u_2 \in U.
    $$
\end{definition}

 \subsection{Clipping Operation}

 For a function $f: \mathbb{R} \rightarrow \mathbb{R}$, 
we define the clipping operation:
\begin{align*}
    \CL_a(f)=\min\{\max\{f,-a\},a\}
\end{align*}
for some $a\geq 0$ . This clipping operation can be realized by a two-layer ReLU network
\begin{align}
    \CL_a(f)=-\ReLU(-\ReLU(f+a)+2a)+a.
\end{align}

\subsection{Notation}
In this paper, we use normal lowercase letters to denote scalars, and bold lowercase letters to denote vectors. Matrices, sets and operators are denoted by upper case letters. We use $U$ to denote the input function set with domain $\Omega_V$, and $V$ to denote the output function set with domain $\Omega_V$. 
We denote the operator to be learned which maps functions in $U$ to functions in $V$ by $G$. Express a $d$-dimensional vector $\xb$ as $\xb=[x_1,...,x_d]^{\top}$.
The $\ell^{\infty}$ and $\ell^2$ norm of a vector $\xb$ is defined $\|\xb\|_{\infty}=\max_k |x_k|$ and $\|\xb\|_2=\sqrt{\sum_k^d x_k^2}$, respectively. 
We denote the Euclidean ball with center $\cbb$ and radius $\delta$ by $\mathcal{B}_{\delta}(\cbb)$.
The $L^{\infty}$ and $L^2$ norm of a function over domain $\Omega_U$ is defined as $\|u\|_{L^{\infty}(\Omega_U)}=\sup_{\xb\in \Omega_U} |u(\xb)|$ and $\|u\|_{L^{2}(\Omega_U)}=\sqrt{\int_{\Omega_U} [u(\xb)]^2d\xb}$, respectively. We define the $\|\cdot\|_{\infty,\infty}$ norm of an operator $G: U\rightarrow V$ by $\|G\|_{\infty,\infty}=\sup_{\yb\in \Omega_V}\sup_{u\in U} |G(u)(\yb)|$.

\section{Problem Setup and Deep Operator Learning}
\label{sec.setup}

\subsection{Problem Setup and Examples}
This paper studies the operator learning problem where the goal is to learn an unknown Lipschitz operator $G: U\rightarrow V$ between two function sets $U$ and $V$ from $n$ training samples $\{(u_i,v_i)\}_{i=1}^n$, where $u_i \in U$ and 
\begin{equation}
v_i=G(u_i)+\zeta_i
\label{eqvi}
\end{equation} with $\zeta_i$ representing noise.
 We consider Lipschitz operators in the following sense:\begin{assumption}\label{assumption_G}
    Let $\Omega_U$ and $\Omega_V$ be the domain of functions in $U$ and $V$ respectively, and $U\subset L^2(\Omega_U)$, $V\subset L^{\infty}(\Omega_V)$. Assume $G: U \rightarrow V$ is a Lipschitz operator: there exists a constant $L_G>0$ such that
    \begin{align*}
        \|G(u_1)-G(u_2)\|_{L^{\infty}(\Omega_V)}\leq L_G\|u_1-u_2\|_{L^2(\Omega_U)},
    \end{align*}
    for any $u_1,u_2\in U$.
\end{assumption}

In Assumption \ref{assumption_G}, the function distance in the output space is measured by the $L^{\infty}$ norm, and the function distance in the input space is measured by the $L^2$ norm. This condition is needed in our network construction to derive an error bound for the branch net.
Assumption \ref{assumption_G} is satisfied for the solution operator of many differential equations. We provide some examples below.

The first example is a nonlinear ODE system known as gravity pendulum with external force, which is studied in \citet{lu2021learning,lanthaler2022error,reid2009pendulum}. 
\begin{example}
\label{ex.gravity}
    Consider the following ODE system 
    \begin{align}
        \begin{cases}
            \frac{dv_1}{dt}=v_2,\\
            \frac{dv_2}{dt}=-\gamma \sin(v_1)+u(t)
        \end{cases}
        \label{eq.gravity}
    \end{align}
    with initial condition $v_1(0)=v_2(0)=0$, and $\gamma>0$ is a parameter.
    In (\ref{eq.gravity}), $v_1,v_2$ represent the angle and angular velocity of the pendulum, $\gamma$ is the frequency parameter and $u(t)$ is an external force controlling the dynamics of the pendulum.
    For this ODE, we consider the operator: $G: u(t)\rightarrow(v_1(t),v_2(t))$. Let $T>0$ the ending time. For any $u_1,u_2\in L^2([0,T])$, there exists a constant $L_G$ such that  
    \begin{align}
        \|G(u_1)-G(u_2)\|_{L^{\infty}([0,T])}\leq L_G\|u_1-u_2\|_{L^2([0,T])}
    \end{align}
    which is proved in  \citet[Proof of Lemma 4.1]{lanthaler2022error}.
\end{example}

In the second example, we consider a transport equation.
\begin{example}
\label{ex.fourier}
Let $\Omega\subset \RR^d$ be a hyper-cube.
    Consider the transport equation on $\Omega \times [0,T]$  
    $$
    \begin{cases}
        v_t=\cbb \cdot \nabla v & \mbox{ on } \Omega \times [0,T]\\
        v(\xb,0)=u(\xb) & \mbox{ on } \Omega 
    \end{cases}
    $$
    equipped with  periodic boundary condition where $\cbb\in \RR^d$ is the velocity. Let $G$ be the solution operator  from the initial condition $u$ to the solution $v(\xb,T)$ at  time $T>0$. We set $\Omega_U=\Omega_V=\Omega$.
    Let $\{w_j\}_{j=1}^J$  be a set of Fourier basis for some positive integer $J>0$, and 
    $$
    U=\left\{\sum_{j=1}^J a_jw_j: \max_j|a_j|\leq C\right\}
    $$
    for some $C>0$. Then Assumption \ref{assumption_G} is satisfied with $L_G=\sqrt{J}$ (see  Section \ref{proof.ex.fourier} for a proof). 
\end{example}

\subsection{Deep Operator Learning}
We study the DeepONet \citep{chen1995universal,lu2021learning} architecture which consists of a branch net and a trunk net. The branch net encodes the input function and produces a set of coefficients. The trunk net learns a set of basis functions for the output  space. A DeepONet takes an input function together with points in the output function domain. 
It outputs a scalar which is the output function evaluated at the given points.

Let $\cF_1=\cF_{\rm NN}(d_{1}, 1, L_1, p_1,K_1, \kappa_1, R_1)$ be the network class for the branch net and $\cF_2 = \cF_{\rm NN}(d_{2}, 1, L_2, p_2,K_2, \kappa_2, R_2)$ be the network class for the trunk net. 
We define the network class of DeepONet as
\begin{align}
    \cG_{\rm NN}=\left\{G_{\rm NN}(\ub )(\yb)=\CL_{\beta_V}\left(\sum_{k=1}^{N} \widetilde{a}_k(\ub )\widetilde{q}_k(\yb)\right): \widetilde{q}_k\in \mathcal{F}_1,\widetilde{a}_k\in \mathcal{F}_2 \mbox{ for } k=1,..,N\right\},
    \label{eq.DON}
\end{align}
where $\ub$ is an input vector, which can be thought as a discretization of the input function, and $\yb$ is a point in the domain of output functions. The network architecture is illustrated in Figure \ref{fig.DON}. A DeepONet takes the discretized function $\ub$ and a point $\yb\in \Omega_V$ as input, where $\ub$ is passed to the branch net $\widetilde{a}_k$'s to compute a set of coefficients, and $\yb$ is passed to the trunk net to evaluate each basis function $\widetilde{q}_k$ at $\yb$. The output $G_{\rm NN}(\ub)(\yb)$ is the sum of the $\widetilde{q}_k$'s value weighted by the coefficients $\widetilde{a}_k$'s from the branch net. Each network in $\cG_{\rm NN}$ is clipped so that its output is bounded by $\beta_V$. This step does not increase the error and is important for our theory.

\begin{figure}
    \centering
    \includegraphics[width=0.9\textwidth]{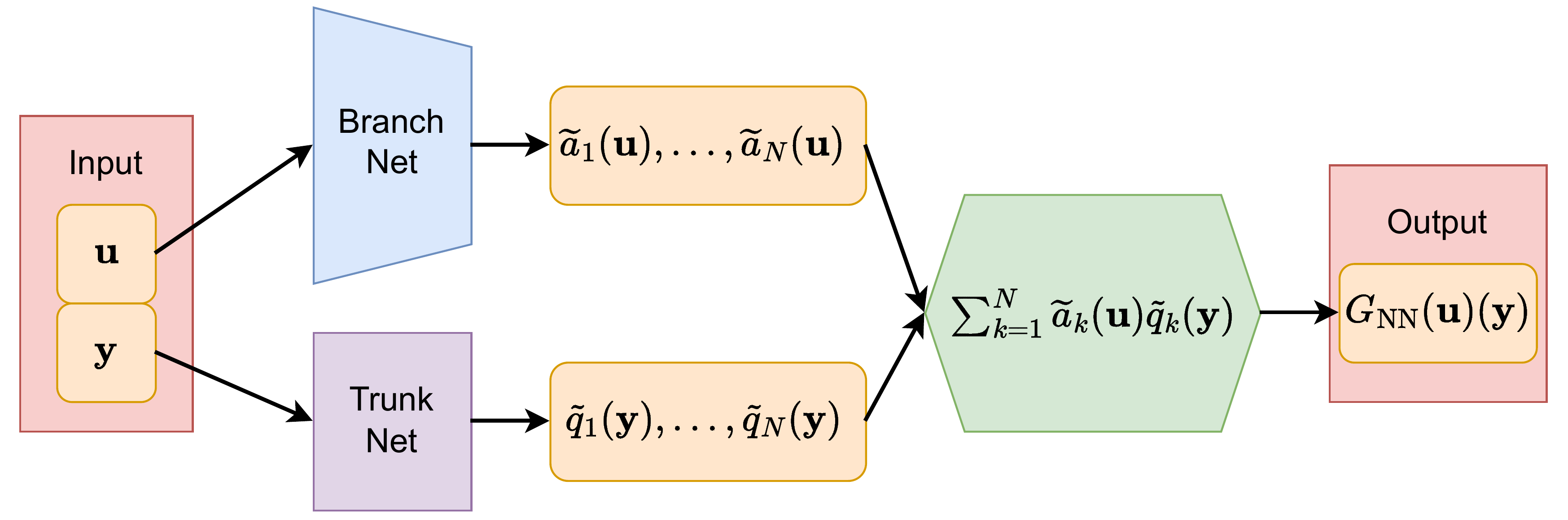}
    \caption{Illustration of the DeepONet architecture. Here $\textbf{u}$ is the discretization of $u\in U$, and $\textbf{y}\in \Omega_V$. }
    \label{fig.DON}
\end{figure}

\section{Main Results}
\label{sec.mainresults}
\subsection{Assumptions}
In this section, we make some assumptions on the function sets $U$ and $V$. 
\begin{assumption}[Input  space $U$]\label{assumption_U}
    Suppose $U$ is a function set such that
    \begin{enumerate}[label=(\roman*)]
        \item Any function  $u\in U$ is defined on
         $\Omega_U=[-\gamma_1, \gamma_1]^{d_1}$ for some $\gamma_1>0$. 
        \item Any function $u\in U$ is Lipschitz with a Lipschitz constant no more than $L_U>0$: $$|u(\xb_1) - u(\xb_2)|\leq L_{U}\|\xb_1 -\xb_2\|_{2}$$ for any $\xb_1,\xb_2\in \Omega_U$.
        \item Any function $u\in U$ satisfies $\|u\|_{L^{\infty}(\Omega_U)}\leq \beta_U$ for some $\beta_U>0$. 
    \end{enumerate}
\end{assumption}

The following assumption is made on the output function set $V$.
\begin{assumption}[Output space $V$]\label{assumption_V}
    Suppose $V$ is a function space such that
    \begin{enumerate}[label=(\roman*)]
        \item Any function in $V$ is defined on $\Omega_V=[-\gamma_2, \gamma_2]^{d_2}$ for some $\gamma_2>0$,
        \item Any function $v\in V$ is Lipschitz with a Lipschitz constant no more than  $L_V>0$: $$|v(\by_1)-v(\by_2)|\leq L_V \|\by_1-\by_2\|_2$$ for any $\by_1,\by_2\in \Omega_V$.
        \item Any function $v\in V$ satisfies $\|v\|_{L^{\infty}(\Omega_V)}\leq \beta_V$ for some $\beta_V>0$.  
    \end{enumerate}       
\end{assumption}

Assumption \ref{assumption_U} and \ref{assumption_V} are mild conditions on $U$ and $V$ and are usually satisfied in applications.

\subsection{DeepONet Approximation Error and Model Scaling Law}
\label{sec.approximation} 
Our first result is on the approximation error of DeepONet for the representation of Lipschitz operators.
\begin{theorem}\label{thm_operator}
Let $d_1,d_2>0$ be integers, $\gamma_1,\gamma_2,\beta_U,\beta_V, L_U,L_V,L_G>0$, and $U,V$ be function sets satisfying Assumption \ref{assumption_U} and \ref{assumption_V} respectively. 
There exist constants $C$ depending on $d_2,L_V,\gamma_2$, $C_F$ depending on $\gamma_1,\beta_1,d_1,L_G$ 
 and $C_{\delta}$ depending on $\gamma_1,\gamma_2,d_1,L_f,L_U$ such that the following holds:
    For any $\varepsilon>0$, set $\delta=C_{\delta}\varepsilon$ and  $N=C\varepsilon^{-d_2}$.  Choose $\{\cbb_m\}_{m=1}^{c_U}\subset \Omega_U$ so that $\{\mathcal{B}_{\delta}(\cbb_m) \}_{ m  = 1}^{c_U}$ is a cover of $\Omega_U$.    Then there exist two network architectures: $\cF_1=\cF_{\rm NN}(d_2,1,L_1,p_1,K_1,\kappa_1,R_1)$  with
    \begin{align*}
        L_1 = O\left(\log(\varepsilon^{-1})\right),\ p_1 = O(1),\ K_1 = O\left(\log(\varepsilon^{-1})\right), \ {\kappa_1=O(\varepsilon^{-d_2-1})},\ R_1=1
    \end{align*}
and  $\cF_2=\cF_{\rm NN}(c_U,1,L_2,p_2,K_2,\kappa_2,R_2)$ with
\begin{align*}
        &L_2 = O\left(c_U^2\log c_U+c_U^2\log(\varepsilon^{-1})\right),\  p_2 = O((C_F\sqrt{c_U})^{c_U}\varepsilon^{-c_U}),\\
&K_2 = O\left(((C_F\sqrt{c_U})^{c_U}\varepsilon^{-c_U})(c_U^2\log c_U+c_U^2\log(\varepsilon^{-1}))\right), \\ &\kappa_2=O(c_U^{c_U/2+1}\varepsilon^{-c_U-1}),\ R=\beta_V.
    \end{align*}
such that, for any operator $G:U\rightarrow V$ satisfying Assumption \ref{assumption_G}, there are $\{\widetilde{q}_k\}_{k=1}^{N}$ with $\widetilde{q}_k  \in \mathcal{F}_1$ and $\{\widetilde{a}_k\}_{k=1}^{N} $ with $\widetilde{a}_k \subset \mathcal{F}_2$ such that 
\begin{align*}
        \sup_{u\in U}\sup_{\yb\in \Omega_V}\left|G(u)(\yb)-\sum_{k=1}^{N} \widetilde{a}_k(\ub )\widetilde{q}_k(\yb)\right|\leq \varepsilon,
    \end{align*}
    where  $\ub=[u(\cbb_1), u(\cbb_2),...,u(\cbb_{c_U})]^{\top}$ is a discretization of $u$.
The constant hidden in $O$ depends on $\gamma_1,\gamma_2,\beta_U,\beta_V,d_1,d_2,L_G,L_U,L_V$.
    
\end{theorem}
Theorem \ref{thm_operator} is proved in Section \ref{proof.operator}.
Theorem \ref{thm_operator} is a general result that holds for any $\{\mathcal{B}_{\delta}(\cbb_m) \}_{ m  = 1}^{c_U}$ covering $\Omega_U$. The following lemma (see a proof in Section \ref{proof.lem.cover.ball}) shows that for any bounded hypercube, there always exists a cover with Euclidean balls, and  an upper bound on the covering number is provided.

\begin{lemma}\label{lem.cover.ball}
    Let $\Omega=[-\gamma,\gamma]^d$ for some $\gamma>0$. For any $\delta>0$, there exists a cover $\{\mathcal{B}_{\delta}(\cbb_m) \}_{ m  = 1}^{M}$ of $\Omega$ with 
    \begin{align}
        \textcolor{black}{M\leq C\delta^{-d}}
        \label{eqn_chapter2}
    \end{align}
    where $C$ is a constant depending on $\gamma$ and $d$.
\end{lemma}

Combining Theorem \ref{thm_operator} and Lemma \ref{lem.cover.ball} yields the following corollary, which quantifies $c_U$ in terms of $\varepsilon $ and $d_1$.
\begin{corollary}\label{coro_operator}
    Let $d_1,d_2>0$ be integers, $\gamma_1,\gamma_2,\beta_U,\beta_V, L_U,L_V>0$, and $U,V$ be function sets satisfying Assumption \ref{assumption_U} and \ref{assumption_V} respectively. There exist constants $C$ depending on $d_2,L_V,\gamma_2$, 
    { $C_F$ depending on $\gamma_1,\beta_1,d_1,L_G$} and $C_{\delta},C_1$ depending on $\gamma_1,d_1,L_f,L_U$, such that the following hold: For any $\varepsilon>0$, set $\delta=C_{\delta}\varepsilon$, $c_U=C_1\varepsilon^{-d_1}$
    and $N=C\varepsilon^{-d_2}$. There exist $\{\cbb_m\}_{m=1}^{c_U}\subset \Omega_U$  so that $\{\mathcal{B}_{\delta}(\cbb_m) \}_{ m  = 1}^{c_U}$ covers $\Omega_U$, and two network architectures: $\cF_1=\cF_{\rm NN}(d_2,1,L_1,p_1,K_1,\kappa_1,R_1)$ and $\cF_2=\cF_{\rm NN}(c_U,1,L_2,p_2,K_2,\kappa_2,R_2)$ with
    \begin{align*}
        L_1 = O\left(\log(\varepsilon^{-1})\right),\ p_1 = O(1),\ K_1 = O\left(\log(\varepsilon^{-1})\right), \ \kappa_1=O({ \varepsilon^{-d_2-1}}),\ R_1=1
    \end{align*}
and 
\begin{align*}
&L_2 = O\left(\varepsilon^{-2d_1}\log \varepsilon^{-1}\right),\  p_2 = O\left({ (C_1C_F)^{C_1\varepsilon^{-d_1}}\varepsilon^{-(d_1/2+1)C_1\varepsilon^{-d_1}}}\right),\\
&K_2 = O\left({ (C_1C_F)^{C_1\varepsilon^{-d_1}}\varepsilon^{-(d_1/2+1)C_1\varepsilon^{-d_1}-2d_1}\log\varepsilon^{-1}}\right), \\ 
&\kappa_2=O\left({ C_1^{C_1\varepsilon^{-d_1}/2+1}\varepsilon^{-(d_1/2+1)C_1\varepsilon^{-d_1}-d_1-1}}\right),\ R=\beta_V,
    \end{align*}
    such that, for any operator $G:U\rightarrow V$ satisfying Assumption \ref{assumption_G}, there are $\{\widetilde{q}_k\}_{k=1}^{N}$ with $\widetilde{q}_k \in \mathcal{F}_1$ and $\{\widetilde{a}_k\}_{k=1}^{N}$ with $\widetilde{a}_k \in \mathcal{F}_2$ such that 
\begin{align*}
        \sup_{u\in U} \sup_{\yb\in \Omega_V}|G(u)(\yb)-\sum_{k=1}^{N} \widetilde{a}_k(\ub )\widetilde{q}_k(\yb)|\leq \varepsilon,
    \end{align*}
    where $\ub=[u(\cbb_1), u(\cbb_2),...,u(\cbb_{c_U})]^{\top}$ is a discretization of $u$.
The constant hidden in $O$ depends on $\gamma_1,\gamma_2,\beta_U,\beta_V,d_1,d_2,L_G,L_U,L_V$.

\end{corollary}
Corollary \ref{coro_operator} can be proved by replacing Theorem \ref{thm_functional} by Corollary \ref{coro_functional} in the proof of Theorem \ref{thm_operator}.
Theorem \ref{thm_operator} and Corollary \ref{coro_operator} have the following implications:

\begin{itemize}
    \item \textbf{Model scaling law.} 
Theorem \ref{thm_operator} and Corollary \ref{coro_operator} show that if the network architecture is properly set, DeepONet can approximate any Lipschitz operator to arbitrary accuracy. To achieve an accuracy $\varepsilon$, the  network size is on the order of $NK_2 =C_F^{\varepsilon^{-d_1}}\varepsilon^{-(d_1/2+C_1)\varepsilon^{-d_1}-2d_1-d_2}\log \varepsilon^{-1}$. In other words, if we denote the total number of network parameters by $N_{\#}=NK_2$, then the network approximation error is in the order of $\left(\frac{\log N_{\#}}{\log\log N_{\#}}\right)^{-1/d_1}$, {\color{black}which critically depends on the input domain dimension $d_1$. This is because the network size is dominated by $\varepsilon^{-\varepsilon^{-d_1}}$ in $NK_2$. This dependence is consistent with existing analysis of curse of dimensionality for learning functions, for which the curse is caused by the input dimension, not the output dimension.} This result gives a theoretical estimation of the model scaling law, which depicts the relation between the network approximation error and the network size. Without making additional assumptions on the low-dimensional structures of the input function set, the network approximation error scales poorly (converges at an extremely slow rate) as the model size increases. In Section \ref{sec.lowd}, we will show that this scaling law can be improved by utilizing low-dimensional structures.

\item \textbf{Optimality.}
In the proof of Theorem \ref{thm_operator}, an important ingredient is to approximate  Lipschitz functionals, which is given in Theorem \ref{thm_functional} and Corollary \ref{coro_functional}. The network size in Theorem \ref{thm_operator} is comparable to that in Corollary \ref{coro_functional}. As discussed in Remark \ref{remark.functional.optimal}, our network size for approximating Lipschitz functionals is optimal up to a logarithmic factor. Since approximating a Lipschitz operator is more difficult than approximating a Lipschitz functional, we expect the network size in Theorem \ref{thm_operator} to be close to the optimal one. Notably, a lower bound of the network complexity of approximating $r$-times Fr\'echet differentiable operators is analyzed in \citet{lanthaler2023curse} for several popular network architectures, including DeepONet. 
The lower bound of the DeepONet size given in \citet[Proposition 2.21]{lanthaler2023curse} for the approximation of Lipschitz functionals is on the order of $\exp(c_1 \varepsilon^{-1/(\alpha+1+\delta)})$ where $\alpha$ is a parameter depending on $d_1$ and $\delta$ is a positive number.

\item \textbf{Connection to existing works.} Approximation theory of DeepONet has  been studied in \citet{chen1995universal} and \citet{lanthaler2022error}. The network size in \citet{chen1995universal} was not explicitly specified, which cannot explain model scaling laws. \citet{lanthaler2022error} conducted an in-depth study of DeepONet, in which a DeepONet is decomposed into three components: an encoder, an approximator and a reconstructor. \citet{lanthaler2022error} analyzed the network structure of each component on several concrete examples. Our settings and results are different from those in \citet{lanthaler2022error} in the following aspects: (i) Our approximation error is measured by the $L^{\infty}$ norm, while \citet{lanthaler2022error} studied the $L^2$ error. (ii) \citet{lanthaler2022error} decomposed the DeepONet approximation error into an encoder error, an approximator error and a reconstructor error, and analyzed each of them. The encoder error and reconstructor error are expressed in terms of the eigenvalues of the covariate operator of the input and output function distributions. An explicit relation between the network size and DeepONet approximation error for general operators was not given. In our paper, we analyze the DeepONet approximation error for general Lipschitz operators and explicitly quantify how the error scales with respect to the network size.

\end{itemize}

\subsection{Generalization Error and Data Scaling Law}
\label{sec.generalization}
Let $n>0$ be a positive integer. Assume we are given the data set $\cS=\{u_i,v_i\}_{i=1}^n$ where $u_i$'s are i.i.d. samples following a distribution $\rho_u$, and $v_i$ is given by \eqref{eqvi}. 
Our setting is summarized below.
\begin{setting}
\label{setting.general}
    Let $\{\xb_{j}\}_{j=1}^{n_x}\subset \Omega_U$ (independent of $i$) be a fixed grid in $\Omega_U$, where $n_x$ is the number of grid points in $\Omega_U$ which is to be specified later. For  $i=1,...,n$, let $ \{\yb_{i,j}\}_{j=1}^{n_y}\subset \Omega_V$ be i.i.d. samples following a distribution  $\rho_y$ on $\Omega_V$. 
Assume we are given a set of paired samples $\{\ub_i,\vb_i\}_{i=1}^n$ with 
\begin{align}
    \ub_i=[u_i(\xb_{1}),...,u_i(\xb_{n_x})]^{\top}, \quad \vb_i=[G(u_i)(\yb_{i,1})+\xi_{i,1},..., G(u_i)(\yb_{i,n_y})+\xi_{i,n_y}]^{\top},
\end{align} 
where $u_i$'s are i.i.d. samples from the  distribution $\rho_u$ in $U$, and $\{\xi_{i,j}\}$ follows i.i.d sub-Gaussian distribution with variance proxy $\sigma^2$. We denote $\bxi_i=[\zeta(\yb_{i,1}),...,\zeta(\yb_{i,n_y})]$. Suppose $U,V$ satisfy Assumption \ref{assumption_U} and \ref{assumption_V}, respectively.
\end{setting}

In Setting \ref{setting.general}, $\{\yb_{i,j}\}_{j=1}^{n_y}$ is the set of discretization grids in $\Omega_V$ for the output function $v_i$. This setting allows output functions in the data set to have different discretization grids in $\Omega_V$.

We consider training DeepONet by minimizing $\frac{1}{nn_y} \sum_{i=1}^n \sum_{j=1}^{n_y}(G_{\rm NN}(\ub_i)(\yb_{i,j})-v_{i,j})^2$ over $G_{\rm NN}\in \cG_{\rm NN}$ to obtain the following minimizer:
\begin{align}
    \widehat{G}=\argmin_{G_{\rm NN}\in \cG_{\rm NN}} \frac{1}{nn_y} \sum_{i=1}^n \sum_{j=1}^{n_y}(G_{\rm NN}(\ub_i)(\yb_{i,j})-v_{i,j})^2,
    \label{eq.gene.loss}
\end{align}
where $v_{i,j} = (\vb_i)_j$ an $\cG_{\rm NN}$ denotes the DeepONet network class given in  \eqref{eq.DON}.
In this paper, we study the squared generalization error of DeepONet given by:
\begin{align*}
  \text{Squared Generalization Error} :=  \EE_{\cS}\EE_{u\sim \rho_u}\EE_{\{\yb_j\}_{j=1}^{n_y}\sim \rho_y} \left[\frac{1}{n_y}\sum_{j=1}^{n_y}(\widehat{G}(\ub)(\yb_j)-G(u)(\yb_j))^2\right],
\end{align*}
where $\bu = [u(\xb_{1}), ..., u(\xb_{n_x})]^{\top}$.
The following theorem gives an upper bound of the generalization error of DeepONet for learning Lipschitz operators.
\begin{theorem}
\label{thm.generalization} 
    Let $d_1,d_2,n_y,n>0$ be integers, $\gamma_1,\gamma_2,\beta_U,\beta_V, L_U,L_V,L_G>0$. Suppose $G:U\rightarrow V$ satisfy Assumption \ref{assumption_G} and consider Setting \ref{setting.general}.
    There exist constants $C$ depending on $d_2,L_V$, $C_F$ depending on $\gamma_1,\beta_1,d_1,L_G$ and $\gamma_2$, $C_1$ depending on $\gamma_1,d_1,L_f,L_U$, $C_2$ depending on $\gamma_1,\gamma_2,\beta_U,\beta_V,d_1,d_2,L_G,L_U,L_V$, and $C_{\delta}$ depending on $\gamma_1,d_1,L_f,L_U$,  such that the following holds: Let  $\varepsilon\in (0,1)$, $\delta=C_{\delta}\varepsilon$ and $N=C\varepsilon^{-d_2}$. Set  $n_x=C_1\varepsilon^{-d_1}$, and then there exist $\{\xb_j \}_{ j  = 1}^{n_x}$ such that  $\{\mathcal{B}_{\delta}(\xb_j) \}_{ j  = 1}^{n_x}$ is a cover of $\Omega_U$.
    Consider the DeepONet network (\ref{eq.DON}) with the network architecture $\cF_1=\cF_{\rm NN}(d_2,1,L_1,p_1,K_1,\kappa_1,R_1)$ and $\cF_2=\cF_{\rm NN}(C\varepsilon^{-C_1\varepsilon^{-d_1}},1,L_2,p_2,K_2,\kappa_2,R_2)$ with
    \begin{align}
        L_1 = O\left(\log(\varepsilon^{-1})\right),\ p_1 = O(1),\ K_1 = O\left(\log(\varepsilon^{-1})\right), \ \kappa_1=O({ \varepsilon^{-d_2-1}}),\ R_1=1
        \label{eq.gene.branch.net}
    \end{align}
and 
\begin{align}
&L_2 = O\left(\varepsilon^{-2d_1}\log \varepsilon^{-1}\right),\  p_2 = O\left({ (C_1C_F)^{C_1\varepsilon^{-d_1}}\varepsilon^{-(d_1/2+1)C_1\varepsilon^{-d_1}}}\right),\nonumber\\
&K_2 = O\left({ (C_1C_F)^{C_1\varepsilon^{-d_1}}\varepsilon^{-(d_1/2+1)C_1\varepsilon^{-d_1}-2d_1}\log\varepsilon^{-1}}\right), \nonumber\\
&\kappa_2=O\left({ C_1^{C_1\varepsilon^{-d_1}/2+1}\varepsilon^{-(d_1/2+1)C_1\varepsilon^{-d_1}/2-d_1-1}}\right),\ R=\beta_V,
\label{eq.gene.trunk.net}
    \end{align}
    where the constant hidden in $O$ depends on $\sigma,\gamma_1,\gamma_2,\beta_U,\beta_V,d_1,d_2,L_G,L_U,L_V$.
    Let $\widehat{G}$ be the minimizer in (\ref{eq.gene.loss}). Then we have
    {\color{black}
    \begin{align}
    \EE_{\cS}\EE_{u\sim \rho_u}\EE_{\{\yb_j\}_{j=1}^{n_y}\sim \rho_y} &\left[\frac{1}{n_y}\sum_{j=1}^{n_y}(\widehat{G}(\bu)(\yb_j)-G(u)(\yb_j))^2\right] \nonumber\\
    &\leq C_2\left(\varepsilon^2+\left(\frac{\sigma^2}{n_y}+1\right)\frac{1}{n}  C_F^{\varepsilon^{-d_1}}\varepsilon^{-(d_1/2+C_1)\varepsilon^{-d_1}-5d_1-d_2}\right)\log^3\frac{1}{\varepsilon}.
    \label{eq.gene.1}
\end{align}
In particular, setting $\varepsilon=\left(\frac{d_1}{{ d_1+2C_1}}\frac{\log n}{\log \log n}\right)^{-\frac{1}{d_1}}$ gives rise to
\begin{align}
    \EE_{\cS}\EE_{u\sim \rho_u} \EE_{\{\yb_j\}_{j=1}^{n_y}\sim \rho_y}\left[\frac{1}{n_y}\sum_{j=1}^{n_y}(\widehat{G}(\ub)(\yb_j)-G(\ub)(\yb_j))^2\right]
    \leq C_3\left(\frac{\log n}{\log \log n}\right)^{-\frac{2}{d_1}}.
\end{align}}
for some $C_3$ depending on $\sigma,\gamma_1,\gamma_2,\beta_U,\beta_V,d_1,d_2,L_G,L_U,L_V$.
\end{theorem}
Theorem \ref{thm.generalization} is proved in Section \ref{proof.generalization}. To prove Theorem \ref{thm.generalization}, we need to carefully perform a bias-variance trade off. The bias term is related with the approximation error of DeepONet.
In practice, one has to train the network to learn the operator from a given data set. The variance term captures the difference between the trained network and the network used in the approximation theory. The network architecture suggested in Theorem \ref{thm.generalization} is a trade-off by balancing the two error terms.

We have the following discussions:
\begin{itemize}
    \item \textbf{Data scaling law.}
Theorem \ref{thm.generalization} shows that to learn  Lipschitz operators, the generalization error of DeepONet decays in a power law of {\color{black}$\log n$}.  This rate is slower than the empirical observations in \cite{lu2021learning,de2022cost}, which suggest a power data scaling law on the order of $n^{-\alpha}$ for some $\alpha>0$. 
This slow decay in our theory comes from the intrinsic difficulty of operator learning in infinite-dimensional spaces. 
Due to the curse of dimensionality, learning an operator is much more difficult than learning a finite-dimensional function, leading to a slower decay of the generalization error.  In the next subsection, we will show that when the input function set $U$ has some low-dimensional structures, we can derive an upper bound that matches the empirical power law observed in \cite{lu2021learning,de2022cost}.

\item \textbf{Effects of $n$ and $n_y$.} {\color{black}The upper bound in Theorem \ref{thm.generalization} critically depends on $n$. As indicated by equation (\ref{eq.gene.1}), increasing $n_y$ reduces the impact of noise.}

\item \textbf{Connection with existing works.} A similar rate was derived in \citet[Section 4.3]{liu2024deep} for the encoder-decoder framework of operator learning. In \citet[Section 4.3]{liu2024deep}, one assumed the input and output functions are $C^s$ functions and Legendre polynomials were used as encoders and decoders. With a fixed grid in $\Omega_U$ and $\Omega_V$, it was shown that the squared
generalization error decays on the order of $(\log n)^{-s/d_1}$, where $n$ is the number of training samples. The generalization error of DeepONet was also analyzed in \citet{lanthaler2022error}, which studied the variance part and did not address the bias-variance trade-off. 
\end{itemize}

\subsection{Utilizing low-dimensional structures}
\label{sec.lowd}
Corollary \ref{coro_operator} and Theorem \ref{thm.generalization} give rise to a slow rate of convergence of DeepONet due to the curse of dimensionality. 
In this subsection, we incorporate low-dimensional structures of input functions and prove a power law convergence which is consistent with empirical observations in \cite{lu2021learning,de2022cost}.
Specifically, we consider the following assumption on low-dimensional structures of $U$:
\begin{assumption}
\label{assumption_lowd}
    Suppose the function set $U$ satisfies
    \begin{enumerate}[label=(\roman*)]
        \item There exists a finite orthonormal basis functions $\{\omega_k\}_{k=1}^{b_U}$ so that  any $u\in U$ can be expressed as
        \begin{align}
            u=\sum_{k=1}^{b_U} \alpha_k \omega_k,\quad \mbox{ with } \quad  \alpha_k=\int_{\Omega_U} u(\xb)\omega_k (\xb) d\xb.
            \label{eq.low.alpha}
        \end{align}
        \item The discretization grid $\{\xb_j\}_{j=1}^{n_x}$  satisfies that: there is a matrix $A\in \RR^{b_U\times n_x}$ such that for any $u\in U$, we have
        \begin{align}
            A\ub= [\alpha_1,...,\alpha_{b_U}]^{\top},
        \end{align}
        where $\ub =[u(\xb_1),\ldots,u(\xb_{n_x})]^T \in \RR^{n_x}$ and the $\alpha_k$'s are the coefficients of $u$ in (\ref{eq.low.alpha}). We denote $C_A=\|A\|_{\infty,\infty}$.
    \end{enumerate}
\end{assumption}
Assumption \ref{assumption_lowd}(i) assumes that the input functions in $U$ live in a $b_U$-dimensional linear subspace. This assumption is commonly used in numerical PDEs. In particular, for some popular bases, such as Legendre polynomials or Fourier bases, Assumption \ref{assumption_lowd}(ii) is satisfied by properly choosing $\{\xb_j\}_{j=1}^{n_x}$:
\begin{itemize}
    \item \textbf{Legendre polynomials.} Let $\{\omega_k\}_{k=1}^{b_U}$ consist of Legendre polynomials up to degree $\mu$ along each dimension. Then $u\omega_k$ is a polynomial with a degree no larger than $2\mu$ along each dimension. By choosing $\{\xb_j\}_{j=1}^{n_x}$ so that along each dimension, the points are quadrature points corresponding to Legendre polynomials of degree $2\mu$, the integral for $\alpha_k$ in (\ref{eq.low.alpha}) can be exactly computed by quadrature rules:
    $$
    \alpha_k=\sum_{j=1}^{n_x} \beta_{j}u(\xb_j)\omega_k(\xb_j),
    $$ where $\{\beta_{j}\}_{j=1}^{n_x}$ are quadrature weights. Assumption \ref{assumption_lowd}(ii) is satisfied by setting $A=[\ab_1,...,\ab_{b_U}]$ with $\ab_k=[\beta_1\omega_k(\xb_1),...,\beta_{n_x}\omega_k(\xb_{n_x})]^{\top}$.

    \item \textbf{Fourier bases.} Let $\{\omega_k\}_{k=1}^{b_U}$ be Fourier bases so that $\omega_k$ has period $2/N_k$ along each dimension for some integer $N_k\geq 1$, i.e., $\omega_k=\prod_{j=1}^{d_1} \omega_{k,j}(x_j) $ with $\omega_{k,j}=\sin\left(\frac{ N_k\pi}{\gamma_1}x_j\right)$ or $\omega_{k,j}=\cos\left(\frac{ N_k\pi}{\gamma_1}x_j\right)$ . One can choose $\{\xb_j\}_{j=1}^{n_x}$ as uniform grids so that the number of grids along each dimension is $\max_k N_k$. We set $A=[\ab_1,...,\ab_{b_U}]$ with $\ab_k=[\omega_k(\xb_1),...,\omega_k(\xb_{n_x})]^{\top}$.
    Note $A$ is independent of $u$ as $\beta_j$ are the quadrature weights.
\end{itemize}

The following theorem provides an approximation theory of DeepONet under the low dimensional structure in Assumption \ref{assumption_lowd}.
\begin{theorem}
\label{thm_operator_lowd}
    Let $d_1,d_2,b_U,n_x>0$ be integers, $\gamma_1,\gamma_2,\beta_U,\beta_V, L_U,L_V,C_A>0$, $U$ satisfy Assumption \ref{assumption_U} and \ref{assumption_lowd}(i), $V$ satisfy Assumption \ref{assumption_V}, and the discretization grids $\{\xb_j\}_{j=1}^{n_x}$ in $\Omega_U$ satisfy Assumption \ref{assumption_lowd}(ii). 
There exist constants $C_{\delta}$ depending on $\gamma_1,d_1,L_f,L_U$ and $C$ depending on $d_2,L_V,\gamma_2$ such that the following holds:
    For any $\varepsilon>0$, set $N=C\varepsilon^{-d_2}$.    There exist two network architectures: $\cF_1=\cF_{\rm NN}(d_2,1,L_1,p_1,K_1,\kappa_1,R_1)$  with
    \begin{align*}
        L_1 = O\left(\log(\varepsilon^{-1})\right),\ p_1 = O(1),\ K_1 = O\left(\log(\varepsilon^{-1})\right), \ \kappa_1=O(\varepsilon^{-1}),\ R_1=1
    \end{align*}
and  $\cF_2=\cF_{\rm NN}(n_x,1,L_2,p_2,K_2,\kappa_2,R_2)$ with 
\begin{align*}
        &L_2 = O\left(\log(\varepsilon^{-1})\right),\  p_2 = O(\varepsilon^{-b_U}),\ K_2 = O\left(\varepsilon^{-b_U}(\log(\varepsilon^{-1})+n_x)\right), \\ &\kappa_2=O(\varepsilon^{-b_U-1}),\ R=\beta_V.
    \end{align*}
such that, for any operator $G:U\rightarrow V$ satisfying Assumption \ref{assumption_G}, there are $\{\widetilde{q}_k\}_{k=1}^{N}$ with $\widetilde{q}_k  \in \mathcal{F}_1$ and $\{\widetilde{a}_k\}_{k=1}^{N} $ with $\widetilde{a}_k \subset \mathcal{F}_2$ such that
\begin{align*}
        \sup_{u\in U}\sup_{\yb\in \Omega_V}\left|G(u)(\yb)-\sum_{k=1}^{N} \widetilde{a}_k(\ub )\widetilde{q}_k(\yb)\right|\leq \varepsilon.
    \end{align*}
The constant hidden in $O$ depends on $\gamma_1,\gamma_2,\beta_U,\beta_V,d_1,d_2,L_G,L_U,L_V,b_U,C_A$. 
\end{theorem}

Theorem \ref{thm_operator_lowd} is proved in Section \ref{proof.operato.lowd}. Importantly, Theorem \ref{thm_operator_lowd} implies a power-law convergence of the DeepONet approximation error.
\begin{itemize}
    \item \textbf{Model scaling law.}  Compared to Corollary \ref{coro_operator}, Theorem \ref{thm_operator_lowd} has a significant improvement on the network size: the number of nonzero parameters is improved from the order of $C_F^{\varepsilon^{-d_1}}\varepsilon^{-(d_1/2+C_1)\varepsilon^{-d_1}-2d_1-d_2}\log\varepsilon^{-1}$ in Corollary \ref{coro_operator} to the order of $ \varepsilon^{-b_U-d_2}$ in Theorem \ref{thm_operator_lowd} for the $C_1$ defined in Corollary \ref{coro_operator} up to logarithmic factors. 
    If we denote the number of network parameters in Theorem \ref{thm_operator_lowd} by $N_{\#}$, the approximation error is on the order of $N_{\#}^{-\frac{1}{b_U+d_2}}$, demonstrating a power model scaling law.
    \item {\color{black}\textbf{Extension to nonlinear low-dimensional structures.} Assumption \ref{assumption_lowd} is a type of linear low-dimensional structures of input functions. An interesting extension is to consider nonlinear structures, such as the input functions locate on a manifold. To characterize this structure one can utilize the concept of Minkowski dimension, as used in \cite{liu2024generalization}. Proof techniques of Theorem \ref{thm_operator_lowd} cannot be directly extended to study the case of nonlinear low-dimensional structures. The proof of Theorem \ref{thm_operator_lowd} uses one network layer to extract the linear latent features. Nonlinear low-dimensional structures are more complicated. Extra efforts are needed to analyze how DeepONet learns and extracts these structures. We leave this direction as our future work.
    }
\end{itemize}

Based on Assumption \ref{assumption_lowd}, we consider the following setting for learning Lipschitz operators under low-dimensional structures of $U$:
\begin{setting}
\label{setting.lowd}
    Let $\{\xb_{j}\}_{j=1}^{n_x}\subset \Omega_U$ be a fixed discretization gird in $\Omega_U$, where $n_x$ is the number of grid points in $\Omega_U$ which is to be specified later.
    For  $i=1,...,n$, let $ \{\yb_{i,j}\}_{j=1}^{n_y}\subset \Omega_V$ be i.i.d. samples following a distribution  $\rho_y$ on $\Omega_V$. 
Assume we are given a set of paired samples $\{\ub_i,\vb_i\}_{i=1}^n$ with 
\begin{align}
    \ub_i=[u_i(\xb_{1}),...,u_i(\xb_{n_x})]^{\top}, \quad \vb_i=[G(u_i)(\yb_{i,1})+\xi_{i,1},..., G(u_i)(\yb_{i,n_y})+\xi_{i,n_y}]^{\top},
\end{align} 
where $u_i$'s are i.i.d. samples from the  distribution $\rho_u$ in $U$, and $\{\xi_{i,j}\}$ follows i.i.d sub-Gaussian distribution with variance proxy $\sigma^2$. We denote $\bxi_i=[\zeta(\yb_{i,1}),...,\zeta(\yb_{i,n_y})]$. Suppose $U$ satisfies Assumption \ref{assumption_U} and \ref{assumption_lowd}(i), $V$ satisfies Assumption \ref{assumption_V}, and $\{\xb_{j}\}_{j = 1}^{n_x}$ satisfies Assumption \ref{assumption_lowd}(ii).
\end{setting}

The generalization error of DeepONet under Setting \ref{setting.lowd} is given in the following theorem:
\begin{theorem}
\label{thm_generalization_lowd}
In Setting \ref{setting.lowd}, let $d_1,d_2,n_x,n_y,n,b_U>0$ be integers, $\gamma_1,\gamma_2,\beta_U,\beta_V, L_U,L_V,L_G,C_A>0$. Suppose $G:U\rightarrow V$ satisfies Assumption \ref{assumption_G}, and consider Setting \ref{setting.lowd}.
    There exist constants $C$ depending on $d_2,L_V,\gamma_2$ and $C_1$ depending on $\sigma,\gamma_1,\gamma_2,\beta_U,\beta_V,d_1,d_2,L_G,L_U,L_V$,  such that the following holds:  
    Consider the DeepONet network (\ref{eq.DON}) with  $N=Cn^{\frac{d_2}{2+{ b_U}+d_2}}$, $\cF_1=\cF_{\rm NN}(d_2,1,L_1,p_1,K_1,\kappa_1,R_1)$ and $\cF_2=\cF_{\rm NN}(n_x,1,L_2,p_2,K_2,\kappa_2,R_2)$ where
    {\color{black}
    \begin{align}
        L_1 = O\left(\log n\right),\ p_1 = O(1),\ K_1 = O\left(\log n\right), \ \kappa_1=O\left({ n^{\frac{1}{2+b_U+d_2}}}\right),\ R_1=1,
        \label{eq.lowd.gene.branch.net}
    \end{align}
and 
\begin{align}
        &L_2 = O\left(\log n\right),\  p_2 = O\left(n^{{ \frac{b_U}{2+b_U+d_2}}}\right),\  K_2 = O\left({ \varepsilon^{\frac{b_U}{2+b_U+d_2}}\log n}+n_x\right), \nonumber\\
        &\kappa_2=O\left({ n^{\frac{b_U+1}{2+b_U+d_2}}}\right),\  R_2=\beta_V.
        \label{eq.lowd.gene.trunk.net}
    \end{align}}
 The constant hidden in $O$ depends on $\sigma,\gamma_1,\gamma_2,\beta_U,\beta_V,d_1,d_2,L_G,L_U,L_V,b_U,C_A$. 
Then $\widehat{G}\in \cG_{NN}$ solving (\ref{eq.gene.loss}) satisfies the following generalization error bound{\color{black}
    \begin{align}
    \EE_{\cS}\EE_{u\sim \rho_u} \EE_{\{\yb_j\}_{j=1}^{n_y}\sim \rho_y}\left[\frac{1}{n_y}\sum_{j=1}^{n_y}(\widehat{G}(\bu)(\yb_j)-G(u)(\yb_j))^2\right]\leq { C_1n^{-\frac{2}{2+b_U+d_2}}\log^3n}.
\end{align}}

\end{theorem}

Theorem \ref{thm_generalization_lowd} is proved in Section \ref{proof.generalization.lowd}. 
We have the following discussion:
\begin{itemize}
    \item \textbf{Data scaling law.} By exploiting  low-dimensional structures of $U$, the squared generalization error is improved from the order of {\color{black}$\left(\log n/\log\log n\right)^{-2/d_1}$} in Theorem \ref{thm.generalization} to the order of {\color{black} $n^{-\frac{2}{2+b_U+d_2}}\log^3n$} in Theorem \ref{thm_generalization_lowd}. 
This power law decay is consistent with the empirical observations in  \citet{lu2021learning,de2022cost}, and our theory provides a rigorous justification of the data scaling law.

    \item \textbf{Adapting to low-dimensional data structures.} By incorporating the low-dimensional structure in Assumption \ref{assumption_lowd}, we can derive a faster rate of convergence in comparison with the general case. In our network construction, we do not need to explicitly know or learn the bases $\{\omega_l\}_{k=1}^{b_U}$ by neural networks. The bases are encoded in some functionals which are to be learned by neural networks. Our results show that deep neural networks are automatically adaptive to low-dimensional data structures. 
\end{itemize}

\section{Proof of main Results}
\label{sec.main_proof}
\subsection{Proof of Theorem \ref{thm_operator}}
\label{proof.operator}
The proof of Theorem \ref{thm_operator} relies on Theorem \ref{thm_function} and Theorem \ref{thm_functional} below. Theorem \ref{thm_function} is an approximation result for Lipschitz functions by deep neural networks.  

\begin{theorem}\label{thm_function}
Let $d_1>0$ be an integer, $\gamma_1,\beta_U,L_U>0$ and $U$ satisfy Assumption \ref{assumption_U}. There exists some constant $C$ depending on $\gamma_1,L_U$ such that the following holds: For any $\varepsilon>0$, set $N = C\sqrt{d_1}\varepsilon^{-1}$.
Let $\{\cbb_k\}_{k=1}^{N^{d_1}}$ be a uniform grid on $\Omega_U$ with spacing $2\gamma_1/N$ along each dimension. There exists a network architecture $\cF_{\rm NN}(d_1, 1, L, p, K, \kappa, R)$ and networks $\{\widetilde{q}_k\}_{k=1}^{N^{d_1}}$ with $\widetilde{q}_k\in \cF_{\rm NN}(d_1, 1, L, p, K, \kappa, R)$, { $\widetilde{q}_k\geq0$} for $k=1,...,N^{d_1}$, such that for any $u\in U$, we have 
\begin{align}
        \left\|u-\sum_{k=1}^{N^{d_1}} u(\cbb_k)\widetilde{q}_k\right\|_{L^{\infty}(\Omega_U)}\leq \varepsilon.
        \label{equation_thm3_not_in_functional_form}
    \end{align}
Such a network architecture has
    \begin{align*}
&L = {O\left(d_1^2\log d_1+d_1^2\log(\varepsilon^{-1})\right)}, p = O(1), K = O\left(d_1^2\log d_1+d_1^2\log(\varepsilon^{-1})\right),\\
& {\kappa=O(d_1^{d_1/2+1}\varepsilon^{-d_1-1})}, R=1.
    \end{align*}
    The constant hidden in $O$ depends on $L_U$ and $\gamma_1$. 
\end{theorem}

Theorem \ref{thm_function} is proved in Section \ref{proof.function}. Theorem \ref{thm_functional} below guarantees the approximation error for Lipschitz functionals.

\begin{figure}[t!]
    \centering
    \includegraphics[width=0.6\textwidth]{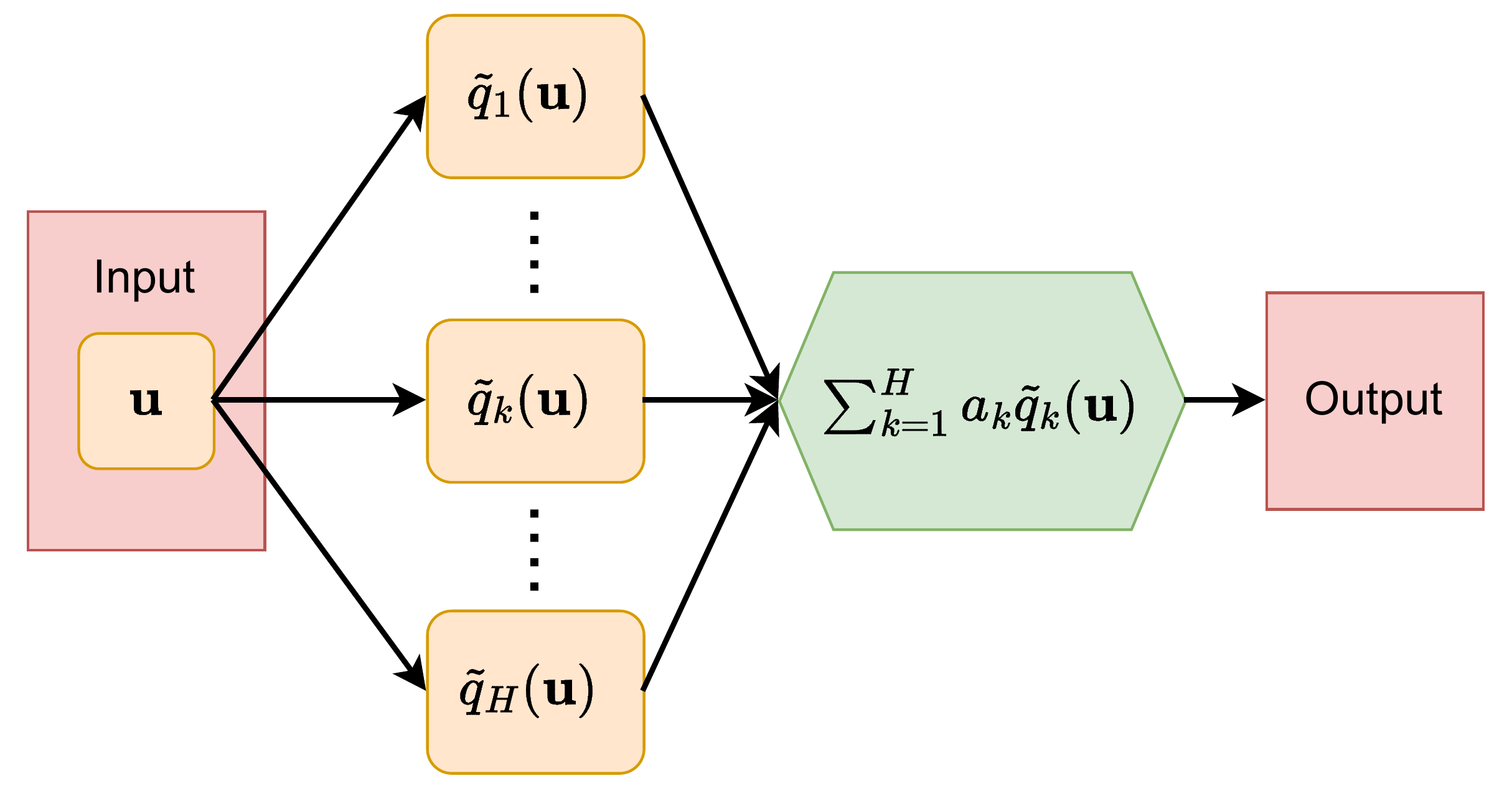}
    \caption{Illustration of the network architecture in Theorem \ref{thm_functional}. Here $\textbf{u}$ is the discretization of $u\in U$. }
    \label{fig.functinoal}
\end{figure}

\begin{theorem}\label{thm_functional}
    Let $d_1>0$ be an integer, $\gamma_1,\beta_U,L_U,L_f,R_f>0$, and $U$ satisfy Assumption \ref{assumption_U}. There exist constants $C$ depending on $\gamma_1,\beta_U,d_1,L_f$ and $C_{\delta}$ depending on $\gamma_1,d_1,L_f,L_U$ such that the following holds:
     For any $\varepsilon>0$, set $\delta=C_{\delta}\varepsilon$ and let $\{\cbb_m\}_{m=1}^{c_U}\subset \Omega_U$ so that $\{\mathcal{B}_{\delta}(\cbb_m) \}_{ m  = 1}^{c_U}$ is a cover of $\Omega_U$ for some $c_U>0$.
     Let {$H=(C\sqrt{c_U})^{c_U}\varepsilon^{-c_U}$}, and set the network $\cF_{\rm NN}(c_U, 1 , L, p, K, \kappa, R)$ 
     with 
     \begin{align*}
        &L = O\left(c_U^2\log c_U+c_U^2\log(\varepsilon^{-1})\right),\ p = O(1), \ K = O\left(c_U^2\log c_U+c_U^2\log(\varepsilon^{-1})\right), \\ 
        &\kappa=O(c_U^{c_U/2+1}\varepsilon^{-c_U-1}), \ R=1.
    \end{align*}
    Denote $\bar{U}_{\beta_U}$ as the set of functions defined on $\Omega_U$ and having range $[-\beta_U,\beta_U]$.
    There are 
     $\{\widetilde{q}_k\}_{k=1}^{H}$ with $ \widetilde{q}_k \in \cF_{\rm NN}(c_U, 1, L, p, K, \kappa,R)$ for any $k$, such that for any Lipschitz functional $f$  on $\bar{U}_{\beta_U}$ with Lipschitz constant $L_f$ for and $\|f\|_{L^{\infty}(\bar{U}_{\beta_U})}\leq R_f$,  we have 
     \begin{align}
        \sup_{u\in \bar{U}_{\beta_U} }\left|f(u)-\sum_{k=1}^{H} a_k\widetilde{q}_k(\ub)\right|\leq \varepsilon,
        \label{eq.functional.1}
    \end{align}
    where $\ub=[u(\cbb_1), u(\cbb_2),...,u(\cbb_{c_U})]^{\top}$, $a_k$'s are coefficients depending on $f$ and satisfying $|a_k|\leq R_f$. 
    The constant hidden in $O$ depends on $\gamma_1,\beta_U,d_1,L_f,L_U$.

\end{theorem}

Theorem \ref{thm_functional} is proved in Section \ref{proof.functional}.  The network architecture is illustrated in Figure \ref{fig.functinoal}. Theorem \ref{thm_functional} expresses the functional network as a sum of $H$ parallel branches, and  the network architecture of each branch is quantified. In the following, we express the functional network as one large network and quantifies the network architecture of this large network. This network can be achieved by stacking the parallel branches together. It has depth similar to that of each $\widetilde{q}$ and width being $H$ times of the width of each $\widetilde{q}_k$, as shown in Figure \ref{fig.functinoal}.

We set the network architecture $\cF_{\rm NN}(c_U,1,L, p, K, \kappa, R)$ 
    as 
    \begin{align}
        &L = O\left(c_U^2\log c_U+c_U^2\log(\varepsilon^{-1})\right),\  p = O({(C\sqrt{c_U})^{c_U}}\varepsilon^{-c_U}),\nonumber\\
        &K = O\left(({ C\sqrt{c_U})^{c_U}}\varepsilon^{-c_U})(c_U^2\log c_U+c_U^2\log(\varepsilon^{-1}))\right), \ \kappa=O(c_U^{c_U/2+1}\varepsilon^{-c_U-1}),\ R=R_f.
        \label{eq:thm6networksize}
    \end{align}
    For any Lipschitz functional $f$ with Lipschitz constant $L_f$ and $\|f\|_{L^{\infty}(U)}\leq R_f$, there exists $\widetilde{f}\in \cF_{\rm NN}(c_U,1,L, p, K, \kappa, R)$ such that 
    \begin{align}
        \sup_{u\in U}|f(u)-\widetilde{f}(\ub)|\leq \varepsilon.
    \end{align}
  The constant hidden in $O$ in \eqref{eq:thm6networksize} depends on $\gamma_1,\beta_U,d_1,L_f,L_U$ and $L_f$.

The following corollary gives an estimation of  $c_U$.
\begin{corollary}\label{coro_functional}
    Let $d_1>0$ be an integer, $\gamma_1,\beta_U,L_U,L_f,R_f>0$, and $U$ satisfy Assumption \ref{assumption_U}. There exist constants $C$ depending on $\gamma_1,\beta_U,d_1,L_f$, and $C_{\delta},C_1$ depending on $\gamma_1,d_1,L_f,L_U$ such that the following holds:
     For any $\varepsilon>0$, set $\delta=C_{\delta}\varepsilon$,  $c_U=C_1\varepsilon^{-d_1}$ and $H={(C\sqrt{c_U})^{c_U}}\varepsilon^{-c_U}$. There exist $\{\cbb_m\}_{ m  = 1}^{c_U}$ such that $\{\mathcal{B}_{\delta}(\cbb_m) \}_{ m  = 1}^{c_U}$ is a cover of $\Omega_U$.  
     Set the network $\cF_{\rm NN}(c_U, 1 , L, p, K, \kappa, R)$ 
     with 
     \begin{align*}
        &L = O\left(c_U^2\log c_U+c_U^2\log(\varepsilon^{-1})\right),\ p = O(1), \ K = O\left(c_U^2\log c_U+c_U^2\log(\varepsilon^{-1})\right), \\ 
        &\kappa=O(c_U^{c_U/2+1}\varepsilon^{-c_U-1}), \ R=1.
    \end{align*}  
    Denote $\bar{U}_{\beta_U}$ as the set of functions defined on $\Omega_U$ and having range $[-\beta_U,\beta_U]$. For any Lipschitz functional $f$ defined on $\bar{U}_{\beta_U}$ with Lipschitz constant $L_f$ and $\|f\|_{L^{\infty}(\bar{U}_{\beta_U})}\leq R_f$, there are
     $\{\widetilde{q}_k\}_{k=1}^{H} $ with $\widetilde{q}_k \in  \cF_{\rm NN}(c_U, 1, L, p, K, \kappa,R)$ such that
     \begin{align}
        \sup_{u\in \bar{U}_{\beta_U}}\left|f(u)-\sum_{k=1}^{H} a_k\widetilde{q}_k(\ub)\right|\leq \varepsilon,
        \label{eq.functional.1_col}
    \end{align}
    where $\ub=[u(\cbb_1), u(\cbb_2),...,u(\cbb_{c_U})]^{\top}$, $a_k$'s are coefficients depending on $f$ and satisfying $|a_k|\leq R_f$. 
    The constant hidden in $O$ depends on $\gamma_1,\beta_U,d_1,L_f,R_f,L_U$.
\end{corollary}
Corollary \ref{coro_functional} is proved in Section \ref{proof.coro_functional}.
\begin{remark}
\label{remark.functional.optimal}
The approximation theory for functionals has been studied in \citet{mhaskar1997neural}.
    It was proved in \citet[Theorem 2.2]{mhaskar1997neural} that, when the activation function is infinitely smooth,   the approximation error of a Lipschitz functional by a two-layer network is lower bounded by $O((\log N)^{-1/d_1})$ where  $N$ is the number of computational neurons. 
    In Corollary \ref{coro_functional}, $N$ is bounded by the total number of weight parameters, thus $N=O\left({ (\varepsilon^{-(d_1/2+1)C_1\varepsilon^{-d_1}}})(\varepsilon^{-2d_1}\log \varepsilon^{-1})\right)$, which implies $\varepsilon=O\left(\left(\frac{\log N}{\log \log N}\right)^{-\frac{1}{d_1}}\right)$. Rewriting (\ref{eq.functional.1_col}) gives
    \begin{align*}
        \left\|f(u)-\sum_{k=1}^{H} a_k\widetilde{q}_k(\ub)\right\|_{L^{\infty}(U)}\leq C_2\left(\frac{\log N}{\log \log N}\right)^{-\frac{1}{d_1}}
    \end{align*}
    for some $C_2$ depending on $\gamma_1,\beta_U,d_1,L_f,R_f,L_U$.
    Our result is consistent with the approximation rate in \cite{song2023approximation} and is optimal up to a $(\log\log N)^{1/d_1}$ factor according to \cite{mhaskar1997neural}.
\end{remark}

\begin{remark}
    A simple set of $\{\cbb_m\}_{m=1}^{c_U}$ satisfying the condition in Corollary \ref{coro_functional} is the uniform grid in $\Omega_U$ with grid spacing $\frac{\varepsilon}{4\sqrt{d_1}(2\gamma_1)^{d_1/2}L_fL_U}$. 
\end{remark}

Now we are ready to prove Theorem \ref{thm_operator}.

\begin{proof}[Proof of Theorem \ref{thm_operator}]
    By Theorem \ref{thm_function}, for any $\varepsilon_1>0$, there exists a constant $N = C\varepsilon_1^{-d_2}$ for some constant  $C$ depending on $d_2,L_V$ and $\gamma_2$, and a network architecture $\mathcal{F}_1=\cF_{\rm NN}(d_2, 1 ,L_1, p_1, K_1, \kappa_1, R_1)$ and  $\{\widetilde{q}_k\}_{k = 1}^{N}$ with $\widetilde{q}_k\in \mathcal{F}_1$, { $\widetilde{q}_k\geq0$} and $\{\cbb_k\}_{k = 1}^{N} \subset \Omega_V$ such that for any $u\in U$, we have
    \begin{align}
        \sup_{\yb\in \Omega_V}\left|G(u)(\yb)-\sum_{k = 1}^{N} G(u)(\cbb_k) \widetilde{q}_k(\yb)\right|\leq \varepsilon_1.
        \label{eq.operator.1}
    \end{align}

    Such a network has parameters
\begin{align*}
        L_1 = O\left(\log(\varepsilon_1^{-1})\right),\ p_1 = O(1),\ K_1 = O\left(\log(\varepsilon_1^{-1})\right), \ \kappa_1=O(\varepsilon_1^{-d_2-1}),\ R_1=1,
    \end{align*}
where the constant hidden in $O$ depends on $d_2,L_V$ and $\gamma_2$.

For each $k$,  \textcolor{black}{define the functional $f_k: V\rightarrow R$ such that}
\begin{equation}
f_k(G(u))=G(u)(\cbb_k).
\label{eq:fk}
\end{equation}
For any $u_1,u_2\in U$, we have $|f_k(G(u_1))|\leq \beta_V$, $|f_k(G(u_2))|\leq \beta_V$ and
\begin{align}
    |f_k(G(u_1))-f_k(G(u_2))|=& |G(u_1)(\cbb_k)-G(u_2)(\cbb_k)| \nonumber \\
    \leq & \sup_{\yb\in \Omega_V}|G(u_1)(\yb)-G(u_2)(\yb)| \nonumber\\
    \leq & \textcolor{black}{L_G\|u_1-u_2\|_{L^2(\Omega_U)},}
    \label{eq.proof.approx.fLip}
\end{align}
where the last inequality follows from Assumption \ref{assumption_G}.

By Theorem \ref{thm_functional}, {let $C_F$ be a constant depending on $\gamma_1,\beta_U,d_1,L_G$} and $C_{\delta}$ be a constant depending on $\gamma_1,d_1,L_f,L_U$. For any $\varepsilon_2>0$, { set $\delta=C_{\delta}'\varepsilon_2$}. There exists  a network architecture $\cF_2=\cF_{\rm NN}(c_U, 1 , L_2, p_2, K_2, \kappa_2, R_2)$ with 
    \begin{align}
        &L_2 = O\left(c_U^2\log c_U+c_U^2\log(\varepsilon_2^{-1})\right),\  p_2 = O({(C_F\sqrt{c_U})^{c_U}}\varepsilon_2^{-c_U}), \nonumber\\
        &K_2 = O\left(({(C_F\sqrt{c_U})^{c_U}}\varepsilon_2^{-c_U})(c_U^2\log c_U+c_U^2\log(\varepsilon^{-1}))\right), \ \kappa_2=O(c_U^{c_U/2+1}\varepsilon_2^{-c_U-1}),\ R=\beta_V,\label{eq:thm1proofnetwork2}
    \end{align}
such that, for every functional $f_k$ defined in \eqref{eq:fk}, this network architecture gives a network $\widetilde{f}_k$ satisfying 
\begin{align*}
    \sup_{u\in U}|f_k(G(u))-\widetilde{f}_k(\ub)|\leq \varepsilon_2.
\end{align*}
The constant hidden in $O$ of \eqref{eq:thm1proofnetwork2} depends on $\gamma_1,\beta_U,\beta_V,d_1,L_G,L_U$.

Since { $0\leq \widetilde{q}_k(\yb)\leq 1$} for any $\yb\in \Omega_V$, we deduce {
\begin{align}
        &\sup_{\yb\in \Omega_V}\left|\sum_{k = 1}^{N} f_k(G(u))\widetilde{q}_k(\yb) - \sum_{k = 1}^{N} \tilde{f 
        }_k(\ub ) \widetilde{q}_k(\yb)\right| \nonumber\\
        =& \sup_{\yb\in \Omega_V}\left|\sum_{k = 1}^{N} \left(f_k(G(u))-\widetilde{f}_k(\ub) \right) \widetilde{q}_k(\yb)\right| \nonumber\\
        \leq& \sup_{\yb\in \Omega_V}\sum_{k = 1}^{N} \left|\left(f_k(G(u))-\widetilde{f}_k(\ub) \right) \right|\widetilde{q}_k(\yb) \nonumber\\
        \leq& \sum_{k = 1}^{N} \left\|f_k(G(u))-\widetilde{f}_k(\ub )\right\|_{L^{\infty}(U)} \widetilde{q}_k(\yb) \nonumber\\
        \leq& \varepsilon_2\sum_{k = 1}^{N}\widetilde{q}_k(\yb). 
        \label{eq.operator.2}
    \end{align}
According to Theorem \ref{thm_function}, for $\{\widetilde{q}_k\}_{k=1}^N$, the error bound (\ref{equation_thm3_not_in_functional_form}) holds for any $v$ in $V$. 
\textcolor{black}{Set $v=\beta_V$}, a constant function. We have 
\begin{align*}
    \sup_{\yb\in \Omega_V}\left| \beta_V-\sum_{k=1}^N \beta_V\widetilde{q}_k(\yb)\right|=\sup_{\yb\in \Omega_V}\left| \beta_V-\beta_V\sum_{k=1}^N \widetilde{q}_k(\yb)\right|=\beta_V\sup_{\yb\in \Omega_V}\left| 1-\sum_{k=1}^N \widetilde{q}_k(\yb)\right|\leq \varepsilon_1,
\end{align*}
implying that
\begin{align}
    \sum_{k=1}^N \widetilde{q}_k(\yb)\leq 1+\varepsilon_1/\beta_V\leq 1+1/\beta_V
    \label{eq.sumq}
\end{align}
when $\varepsilon_1<1$. Substituting (\ref{eq.sumq}) into (\ref{eq.operator.2}) gives rise to
\begin{align}
    \sup_{\yb\in \Omega_V}\left|\sum_{k = 1}^{N} f_k(G(u))\widetilde{q}_k(\yb) - \sum_{k = 1}^{N} \tilde{f 
        }_k(\ub ) \widetilde{q}_k(\yb)\right|\leq (1+1/\beta_V)\varepsilon_2.
\end{align}
}
 Putting (\ref{eq.operator.1}) and (\ref{eq.operator.2}) together, we have
    \begin{align*}
        &\sup_{u\in U} \sup_{\yb\in \Omega_V}\left|G(u)(\yb)-\sum_{k = 1}^{N} \widetilde{f}_k(\ub )\widetilde{q}_k(\yb)\right| \nonumber \\
        \leq & \sup_{u\in U} \sup_{\yb\in \Omega_V}\left|G(u)(y)-\sum_{k = 1}^{N} f_k(G(u))\widetilde{q}_k(y)\right| + \sup_{u\in U} \sup_{\yb\in \Omega_V}\left|\sum_{k = 1}^{N} f_k(G(u))\widetilde{q}_k (y)-\sum_{k = 1}^{N} \widetilde{a}_k(\ub )\widetilde{q}_k(y)\right|\\
        \leq& \varepsilon_1 + (1+1/\beta_V)\varepsilon_2.
    \end{align*}
    Set $\varepsilon_2 = \frac{\varepsilon}{2(1+1/\beta_V)},\varepsilon_1=\frac{\varepsilon}{2}$, we have { $\delta=C_{\delta}\varepsilon$ and }
    \begin{align*}
        &\sup_{u\in U} \sup_{\yb\in \Omega_V}\left|G(u)(\yb)-\sum_{k=1}^{N} \widetilde{a}_k (\ub)\widetilde{q}_k(\yb)\right|\leq \varepsilon,
    \end{align*}
    {here $C_{\delta}$ is a constant depending on $\gamma_1,\gamma_2,d_1,L_f,L_U$}.
The resulting network architectures have $N=O(\varepsilon^{-d_2})$, 
\begin{align*}
        L_1 = O\left(\log(\varepsilon^{-1})\right),\ p_1 = O(1),\  K_1 = O\left(\log(\varepsilon^{-1})\right),\  \kappa_1=O({\varepsilon^{-d_2-1}}),\ R_1=1,
    \end{align*}
and 
\begin{align*}
&L_2 = O\left(c_U^2\log c_U+c_U^2\log(\varepsilon^{-1})\right),\  p_2 = O({(C_F\sqrt{c_U})^{c_U}\varepsilon^{-c_U}}),\\
&K_2 = O\left(((C_F\sqrt{c_U})^{c_U}\varepsilon^{-c_U})(c_U^2\log c_U+c_U^2\log(\varepsilon^{-1}))\right), \\ &\kappa_2=O(c_U^{c_U/2+1}\varepsilon^{-c_U-1}),\ R=\beta_V.
    \end{align*}
The constant hidden in $O$ depends on $\gamma_1,\gamma_2,\beta_U,\beta_V,d_1,d_2,L_G,L_U,L_V$.
\end{proof}

\subsection{Proof of Theorem \ref{thm.generalization}}
\label{proof.generalization}

\begin{proof}[Proof of Theorem \ref{thm.generalization}]
    We rewrite the error as
    \begin{align}
        &\EE_{\cS}\EE_{u\sim \rho_u} \EE_{\{\yb_j\}_{j=1}^{n_y}\sim \rho_y}\left[\frac{1}{n_y}\sum_{j=1}^{n_y}(\widehat{G}(\ub)(\yb_j) - G(u)(\yb_j))^2\right] \nonumber\\
        =&\underbrace{2\EE_{\cS}\left[\frac{1}{n}\sum_{i=1}^n \frac{1}{n_y}\sum_{j=1}^{n_y}(\widehat{G}(\ub_i)(\yb_{i,j})-G(u_i)(\yb_{i,j}))^2\right]}_{\rm T_1} \nonumber\\
        &\underbrace{\begin{aligned}
            +\EE_{\cS}\EE_{u\sim \rho_u} \EE_{\{\yb_j\}_{j=1}^{n_y}\sim \rho_y}&\left[\frac{1}{n_y}\sum_{j=1}^{n_y}(\widehat{G}(\ub)(\yb_j)-G(u)(\yb_j))^2\right]\\
            &-2\EE_{\cS}\left[\frac{1}{n}\sum_{i=1}^n\frac{1}{n_y}\sum_{j=1}^{n_y}(\widehat{G}(\ub_i)(\yb_{i,j})-G(u_i)(\yb_{i,j}))^2\right],
        \end{aligned}}_{\rm T_2}
        \label{eq.gene.deco}
    \end{align}
where $u_i$ and $\textbf{y}_{i, j}$ are given in the training dataset $\mathcal{S}$. 

\noindent $\bullet$ Bounding ${\rm T_1}$. For ${\rm T_1}$, we have
    \begin{align}
        {\rm T_1}=&2\EE_{\cS}\left[\frac{1}{n}\sum_{i=1}^n \frac{1}{n_y}\sum_{j=1}^{n_y}(\widehat{G}(\ub_i)(\yb_{i,j})-G(u)(\yb_{i,j}))^2\right] \nonumber\\
        =&2\EE_{\cS}\left[\frac{1}{n}\sum_{i=1}^n \frac{1}{n_y}\sum_{j=1}^{n_y}(\widehat{G}(\ub_i)(\yb_{i,j})-v_{i,j}+\xi_{i,j})^2\right] \nonumber\\
        =&2\EE_{\cS}\left[\frac{1}{n}\sum_{i=1}^n \frac{1}{n_y}\sum_{j=1}^{n_y}\left[(\widehat{G}(\ub_i)(\yb_{i,j})-v_{i,j})^2+2(\widehat{G}(\ub_i)(\yb_{i,j})-v_{i,j})\xi_{i,j}+\xi_{i,j}^2\right]\right] \nonumber\\
        =&2\EE_{\cS}\left[\frac{1}{n}\sum_{i=1}^n \frac{1}{n_y}\sum_{j=1}^{n_y}(\widehat{G}(\ub_i)(\yb_{i,j}))-v_{i,j})^2\right] \nonumber\\
        &+4\EE_{\cS}\left[\frac{1}{n}\sum_{i=1}^n \frac{1}{n_y}\sum_{j=1}^{n_y}(\widehat{G}(\ub_i)(\yb_{i,j})-v_{i,j})\xi_{i,j}\right] + 2\EE_{\cS}\left[\frac{1}{n}\sum_{i=1}^n \frac{1}{n_y}\sum_{j=1}^{n_y}\xi_{i,j}^2\right] \nonumber\\
         =&2\EE_{\cS}\inf_{G_{\rm NN}\in \cG_{\rm NN}}\left[\frac{1}{n}\sum_{i=1}^n \frac{1}{n_y}\sum_{j=1}^{n_y}(G_{\rm NN}(\ub_i)(\yb_{i,j})-v_{i,j})^2\right] \nonumber\\
         &+4\EE_{\cS}\left[\frac{1}{n}\sum_{i=1}^n \frac{1}{n_y}\sum_{j=1}^{n_y}(\widehat{G}(\ub_i)(\yb_{i,j})-G(u_i)(\yb_{i,j})-\xi_{i,j})\xi_{i,j}\right]+2\EE_{\cS}\left[\frac{1}{n}\sum_{i=1}^n \frac{1}{n_y}\sum_{j=1}^{n_y}\xi_{i,j}^2\right] \nonumber\\
         \leq &2\inf_{\textcolor{black}{G_{\rm NN}}\in \cG_{\rm NN}}\EE_{\cS}\left[\frac{1}{n}\sum_{i=1}^n \frac{1}{n_y}\sum_{j=1}^{n_y}(G_{\rm NN}(\ub_i)(\yb_{i,j})-v_{i,j})^2\right] \nonumber\\
         &+4\EE_{\cS}\left[\frac{1}{n}\sum_{i=1}^n \frac{1}{n_y}\sum_{j=1}^{n_y}(\widehat{G}(\ub_i)(\yb_{i,j}) - G(u_i)(\yb_{i,j}))\xi_{i,j}\right]-2\EE_{\cS}\left[\frac{1}{n}\sum_{i=1}^n \frac{1}{n_y}\sum_{j=1}^{n_y}\xi_{i,j}^2\right] \nonumber\\
         =& 2\inf_{G_{\rm NN}\in \cG_{\rm NN}}\EE_{\cS}\left[\frac{1}{n}\sum_{i=1}^n \frac{1}{n_y}\sum_{j=1}^{n_y}\left((G_{\rm NN}(\ub_i)(\yb_{i,j}) - G(u_i)(\yb_{i,j})-\xi_{i,j})^2-\xi_{i,j}^2\right)\right] \nonumber\\
         & + 4\EE_{\cS}\left[\frac{1}{n}\sum_{i=1}^n \frac{1}{n_y}\sum_{j=1}^{n_y}(\widehat{G}(\ub_i)(\yb_{i,j}) - G(u_i)(\yb_{i,j}))\xi_{i,j}\right] \nonumber\\
         = & 2\inf_{G_{\rm NN}\in \cG_{\rm NN}}\EE_{u\sim \rho_u}\EE_{\textcolor{black}{\{\yb_j\}_{j=1}^{n_y}}\sim \rho_y}\left[ \frac{1}{n_y}\sum_{j=1}^{n_y}(G_{\rm NN}(\ub)(\yb_{j})-G(u)(\yb_{j}))^2\right] \nonumber\\
         &+4\EE_{\cS}\left[\frac{1}{n}\sum_{i=1}^n \frac{1}{n_y}\sum_{j=1}^{n_y}\widehat{G}(\ub_i)(\yb_{i,j})\xi_{i,j}\right]
         \label{eq.gene.T1}
    \end{align}
    The first term in (\ref{eq.gene.T1}) can be bounded by Corollary \ref{coro_operator}. More specifically, {let $C_F$ be a constant depending on $\gamma_1,\beta_1,d_1,L_G$}. For any $\varepsilon>0$, we choose $n_x=C_1\varepsilon^{-d_1}$ so that $\{\mathcal{B}_{\delta}(\cbb_m) \}_{ m  = 1}^{n_x}$ is a cover of $\Omega_U$ with $\delta=C_{\delta}\varepsilon$ for some $C_1,C_{\delta}$ depending on $\gamma_1,d_1,L_f$ and $L_U$. According to Corollary \ref{coro_operator}, we can set network architecture $\cF_1=\cF_{\rm NN}(d_2,1,L_1,p_1,K_1,\kappa_1,R_1)$ and $\cF_2=\cF_{\rm NN}(n_x,1,L_2,p_2,K_2,\kappa_2,R_2)$ with
    \begin{align}
        L_1 = O\left(\log(\varepsilon^{-1})\right),\ p_1 = O(1),\ K_1 = O\left(\log(\varepsilon^{-1})\right), \ \kappa_1=O({\varepsilon^{-d_2-1}}),\ R_1=1
        \label{eq.gene.branch.net.1}
    \end{align}
and 
\begin{align}
&L_2 = O\left(\varepsilon^{-2d_1}\log \varepsilon^{-1}\right),\  p_2 = O\left({ (C_1C_F)^{C_1\varepsilon^{-d_1}}\varepsilon^{-(d_1/2+1)C_1\varepsilon^{-d_1}}}\right),\nonumber\\
&K_2 = O\left({ (C_1C_F)^{C_1\varepsilon^{-d_1}}\varepsilon^{-(d_1/2+1)C_1\varepsilon^{-d_1}-2d_1}\log\varepsilon^{-1}}\right), \nonumber\\ 
&\kappa_2=O\left({C_1^{C_1\varepsilon^{-d_1}/2+1}\varepsilon^{-(d_1/2+1)C_1\varepsilon^{-d_1}-d_1-1}}\right),\ R=\beta_V,
\label{eq.gene.trunk.net.1}
    \end{align}
    The constant in $O$ depends on $\gamma_1,\gamma_2,\beta_U,\beta_V,d_1,d_2,L_G,L_U,L_V$.
    Set $N=C\varepsilon^{-d_2}$ for some constant $C$ depending on $d_2,L_V$ and $\gamma_2$. There are $\{\widetilde{q}_k\}_{k=1}^{N}$ with $\widetilde{q}_k \in  \mathcal{F}_1$ and $\{\widetilde{a}_k\}_{k=1}^{N} $ with $\widetilde{a}_k \in \mathcal{F}_2$ such that
\begin{align*}
        \sup_{u\in U}\sup_{\yb\in \Omega_V}\left|G(u)(\yb)-\sum_{k=1}^{N} \widetilde{a}_k(\ub )\widetilde{q}_k(\yb)\right|\leq \varepsilon.
    \end{align*}
    The first term in (\ref{eq.gene.T1}) is bounded by
    \begin{align}
        &2\inf_{G_{\rm NN}\in \cG_{\rm NN}}\EE_{u\sim \rho_{u}}\EE_{\{\yb_j\}_{j=1}^{n_y}}\left[ \frac{1}{n_y}\sum_{j=1}^{n_y}\left(G_{\rm NN}(\ub)(\yb_{j}) - G(u)(\yb_{j})\right)^2\right] \leq 2\varepsilon^2.
        \label{eq.gene.T1.1}
    \end{align}
The second term in (\ref{eq.gene.T1}) is bounded by the following lemma (see proof in Section \ref{proof.lem.gene.T1.2}):
\begin{lemma}
\label{lem.gene.T1.2}
    Denote $\|G_{\rm NN}\|_n^2=\frac{1}{n}\sum_{i=1}^n \frac{1}{n_y}\sum_{j=1}^{n_y}|G_{\rm NN}(\ub_i)(\yb_{i,j})|^2$. Under the condition of Theorem \ref{thm.generalization}, the second term in (\ref{eq.gene.T1}) is bounded as
    \begin{align}
    &\EE_{\cS}\left[\frac{1}{n}\sum_{i=1}^n \frac{1}{n_y}\sum_{j=1}^{n_y}\widehat{G}(\ub_i)(\yb_{i,j})\xi_{i,j}\right] \nonumber\\
	\leq & 2\sigma\left(\sqrt{\EE_{\cS}\left[\|\widehat{G} - G\|_n^2\right] } +\theta \right) \sqrt{\frac{4 \log \cN(\theta,\cG_{\rm NN},\|\cdot\|_{\infty,\infty}) + 6}{nn_y} } + \sigma\theta.
 \label{eq.gene.T1.2.lem}
\end{align}
\end{lemma}

Let $\hat{G}$ be the network specified in (\ref{eq.gene.branch.net.1}) and (\ref{eq.gene.trunk.net.1}). Substituting (\ref{eq.gene.T1.1}) and (\ref{eq.gene.T1.2.lem}) into (\ref{eq.gene.T1}) gives rise to
\begin{align}
    {\rm T_1}=&2\EE_{\cS}\left[\frac{1}{n}\sum_{i=1}^n \frac{1}{n_y}\sum_{j=1}^{n_y}(\widehat{G}(\ub_i)(\yb_{i,j}) -  G(u_i)(\yb_{i,j}))^2\right] \nonumber\\
    \leq& 2\varepsilon^2+8 \sigma\left(\sqrt{\EE_{\cS}\left[\|\widehat{G}-G\|_n^2\right] } +\theta \right) \sqrt{\frac{4 \log \cN(\theta,\cG_{\rm NN}, \|\cdot\|_{\infty,\infty}) + 6}{nn_y} } + 4\sigma\theta.
    \label{eq.T1.err.1}
\end{align}

Denote
\begin{align*}
    &\eta=\sqrt{\EE_{\cS}\left[\|\widehat{G}-G\|_n^2\right] },\\
    &a=\varepsilon^2+2\sigma\theta+4\sigma\theta\sqrt{\frac{4 \log \cN(\theta,\cG_{\rm NN},\|\cdot\|_{\infty,\infty}) + 6}{nn_y} },\\
    &b=2\sigma\sqrt{\frac{4 \log \cN(\theta, \cG_{\rm NN},\|\cdot\|_{\infty,\infty}) + 6}{nn_y} }.
\end{align*}
Then (\ref{eq.T1.err.1}) can be rewritten as
\begin{align*}
    \eta^2\leq a+2b\eta,
\end{align*}
from which we deduce that 
\begin{align*}
    (\eta-b)^2\leq a+b^2 \quad \Rightarrow \quad \eta\leq \sqrt{a+b^2}+b \quad \Rightarrow \quad \eta^2\leq 2a+4b^2.
\end{align*}
Thus we have
\begin{align}
    {\rm T_1}=&2\eta^2 \nonumber\\
    \leq& 4\varepsilon^2+8\sigma\theta+16\sigma\theta\sqrt{\frac{4 \log \cN(\theta,\cG_{\rm NN},\|\cdot\|_{\infty,\infty}) + 6}{nn_y} }+ 16\sigma^2\frac{4 \log \cN(\theta,\cG_{\rm NN},\|\cdot\|_{\infty,\infty}) + 6}{nn_y}.
    \label{eq.gene.T1.err}
\end{align}

\noindent $\bullet$ Bounding ${\rm T_2}$.
An upper bound of ${\rm T_2}$ is given by the following lemma (see a proof in Section \ref{proof.lem.gene.T2.err})
\begin{lemma}
\label{lem.gene.T2.err}
    Under the condition of Theorem \ref{thm.generalization}, we have
    {\color{black}
    \begin{align}
	{\rm T}_2\leq \frac{19\beta_V^2}{n}\log \cN\left(\frac{\theta}{4\beta_V},\cG_{\rm NN},\|\cdot\|_{\infty,\infty}\right)+6\theta.
	\label{eq.gene.T2.err.lem}
\end{align}}
\end{lemma}

\noindent $\bullet$ Putting ${\rm T_1},{\rm T_2}$ together.

Substituting (\ref{eq.gene.T1.err}) and (\ref{eq.gene.T2.err}) into (\ref{eq.gene.deco}) gives rise to
{\color{black}
\begin{align}
	&\EE_{\cS}\EE_{u\sim \rho_u} \EE_{\{\yb_j\}_{j=1}^{n_y}\sim \rho_y}\left[\frac{1}{n_y}\sum_{j=1}^{n_y}(\widehat{G}(\ub;\yb_j))-G(\ub;\yb_j))^2\right] \nonumber\\
	\leq& 4\varepsilon^2+8\sigma\theta+16\sigma\theta\sqrt{\frac{4 \log \cN(\theta,\cG_{\rm NN},\|\cdot\|_{\infty,\infty}) + 6}{nn_y} }+ 16\sigma^2\frac{4 \log \cN(\theta,\cG_{\rm NN},\|\cdot\|_{\infty,\infty}) + 6}{nn_y} \nonumber\\
	&+ \frac{19\beta_V^2}{n}\log \cN\left(\frac{\theta}{4\beta_V},\cG_{\rm NN},\|\cdot\|_{\infty,\infty}\right)+6\theta \nonumber\\
	\leq& 4\varepsilon^2 +     \left(\frac{64\sigma^2+6}{nn_y}+\frac{19\beta_V^2}{n} \right) \log \cN\left(\frac{\theta}{4\beta_V},\cG_{\rm NN},\|\cdot\|_{\infty,\infty}\right) \nonumber\\
	&+16\sigma\theta\sqrt{\frac{4 \log \cN(\frac{\theta}{4\beta_V},\cG_{\rm NN},\|\cdot\|_{\infty,\infty}) + 6}{nn_y} }+ (8\sigma+6)\theta.
 \label{eq.gene.err.1}
\end{align}}
The following lemma (see a proof in Section \ref{proof.covering}) gives an upper bound of $\cN(\theta,\cG_{\rm NN},\|\cdot\|_{\infty,\infty})$:
\begin{lemma}\label{lem.covering}
    Let $\cF_1=\cF_{\rm NN}(d_1,1,L_1,p_1,K_1,\kappa_1,R_1)$ and $\cF_2=\cF_{\rm NN}(d_2,1,L_2,p_2,K_2,\kappa_2,R_2)$. We have
    \begin{align*}
        \cN(\theta,\cG_{\rm NN},\|\cdot\|_{\infty,\infty})\leq \left(\frac{2L_1p_1^2\kappa_1H}{\theta}\right)^{NK_1}\left(\frac{2L_2p_2^2\kappa_2H}{\theta}\right)^{NK_2},
    \end{align*}
    with $H=N\left(R_1L_1(p_1\gamma_2+2)(\kappa_1p_1)^{L_1-1}+R_2L_2(p_2\beta_U+2)(\kappa_2p_2)^{L_2-1} \right)$.
\end{lemma}

Substituting (\ref{eq.gene.branch.net.1}) and (\ref{eq.gene.trunk.net.1}) into Lemma \ref{lem.covering} gives rise to
\begin{align}
    &\log \cN(\theta,\cG_{\rm NN},\|\cdot\|_{\infty,\infty}) \nonumber\\
    = & NK_1(\log 2+\log L_1+2\log p_1+\log \kappa_1+\log H+ \log \frac{1}{\theta}) \nonumber \\
    &+  NK_2(\log 2+\log L_2+2\log p_2+\log \kappa_2+\log H+ \log \frac{1}{\theta}) \nonumber\\
    \leq &C_4N(K_1+K_2)(\log H + \log \frac{1}{\theta})  \nonumber\\
    \leq & C_4N(K_1+K_2)(\log N+L_2(\log L_2+ \log p_2+\log \kappa_2)+\log\frac{1}{\theta} ) \nonumber \\
    \leq& { C_4(C_1C_F)^{C_1\varepsilon^{-d_1}}\varepsilon^{-(d_1/2+1)C_1\varepsilon^{-d_1}-2d_1-d_2}\log \frac{1}{\varepsilon}\left( \log \frac{1}{\varepsilon}+ \varepsilon^{-2d_1}\log\frac{1}{\varepsilon} \left(\log \frac{1}{\varepsilon}+ \varepsilon^{-d_1}\log \frac{1}{\varepsilon}\right)+\log \frac{1}{\theta}\right) }\nonumber\\
    \leq& { C_4(C_1C_F)^{C_1\varepsilon^{-d_1}}\varepsilon^{-(d_1/2+1)C_1\varepsilon^{-d_1}-5d_1-d_2}\log \frac{1}{\varepsilon}(\log^2 \frac{1}{\varepsilon}+ \log \frac{1}{\theta})}
    \label{eq.gene.covering.bound}
\end{align}
where $C_4$ is a constant depending on $\gamma_1,\gamma_2,\beta_U,\beta_V,d_1,d_2,L_G,L_U,L_V$.

{\color{black}Substituting (\ref{eq.gene.covering.bound}) into (\ref{eq.gene.err.1}) gives rise to
\begin{align}
    &\EE_{\cS}\EE_{u\sim \rho_u}\EE_{\{\yb_j\}_{j=1}^{n_y}\sim \rho_y} \left[\frac{1}{n_y}\sum_{j=1}^{n_y}(\widehat{G}(\ub)(\yb_j) - G(u)(\yb_j))^2\right] \nonumber\\
	\leq& 4\varepsilon^2 + \left(\frac{64\sigma^2+6}{nn_y}+\frac{19\beta_V^2}{n} \right){C_4(C_1C_F)^{C_1\varepsilon^{-d_1}}\varepsilon^{-(d_1/2+1)C_1\varepsilon^{-d_1}-5d_1-d_2}\log \frac{1}{\varepsilon}(\log^2 \frac{1}{\varepsilon}+ \log \frac{1}{\theta})} \nonumber\\
	&+16\sigma\theta\sqrt{\frac{4 {C_4(C_1C_F)^{C_1\varepsilon^{-d_1}}\varepsilon^{-(d_1/2+1)C_1\varepsilon^{-d_1}-5d_1-d_2}\log \frac{1}{\varepsilon}(\log^2 \frac{1}{\varepsilon}+ \log \frac{1}{\theta})} + 6}{nn_y} }+ (8\sigma+6)\theta.
\end{align}
Set $\theta=n^{-1}$. We have
\begin{align}
    \EE_{\cS}\EE_{u\sim \rho_u}\EE_{\{\yb_j\}_{j=1}^{n_y}\sim \rho_y} &\left[\frac{1}{n_y}\sum_{j=1}^{n_y}(\widehat{G}(\ub)(\yb_j) - G(\ub)(\yb_j) )^2\right] \nonumber\\
    &\leq C_2\left(\varepsilon^2+\left(\frac{\sigma^2}{n_y}+1\right)\frac{1}{n} {(C_1C_F)^{C_1\varepsilon^{-d_1}}\varepsilon^{-(d_1/2+1)C_1\varepsilon^{-d_1}-5d_1-d_2}}\right)\log^3\frac{1}{\varepsilon},
    \label{eq.gene.err.eps}
\end{align}
for some $C_2$ depending on $\gamma_1,\gamma_2,\beta_U,\beta_V,d_1,d_2,L_G,L_U,L_V$

In particular, set $\varepsilon=\left(\frac{2d_1}{{(d_1+2)C_1}}\frac{\log n}{\log \log n}\right)^{-\frac{1}{d_1}}$. 
After taking logarithm, $\varepsilon^2$ is of $O(-\log\log n)$, $\left(\varepsilon^2+\left(\frac{\sigma^2}{n_y}+1\right)\frac{1}{n} {(C_1C_F)^{C_1\varepsilon^{-d_1}}\varepsilon^{-(d_1/2+1)C_1\varepsilon^{-d_1}-5d_1}}\right)$ is of $O(-\log n)$. 
Thus the error inside the parenthesis of (\ref{eq.gene.err.eps}) is dominated by $\varepsilon^2$.
We have
\begin{align}
    \EE_{\cS}\EE_{u\sim \rho_u}\EE_{\{\yb_j\}_{j=1}^{n_y}\sim \rho_y} \left[\frac{1}{n_y}\sum_{j=1}^{n_y}(\widehat{G}(\ub)(\yb_j) - G(\ub)(\yb_j))^2\right]\leq C_3\left(\frac{\log n}{\log \log n}\right)^{-\frac{2}{d_1}}.
\end{align}
for some constant $C_3$ depending on $\sigma,\gamma_1,\gamma_2,\beta_U,\beta_V,d_1,d_2,L_G,L_U,L_V$.}
\end{proof}

\subsection{Proof of Theorem \ref{thm_operator_lowd}}
\label{proof.operato.lowd}
To prove Theorem \ref{thm_operator_lowd}, we need the following approximation result for functionals defined on $U$ satisfying Assumption \ref{assumption_lowd}.
\begin{theorem}
\label{thm_functional_lowd}
    Let $d_1,n_x,b_U>0$ be integers, $\gamma_1,\beta_U,L_U,L_f,R_f>0$, and $U$ satisfy Assumption \ref{assumption_U} and \ref{assumption_lowd} (i), the discretization grids $\{\xb_j\}_{j=1}^{n_x}$ satisfy Assumption \ref{assumption_lowd} (ii). There exist constant $C$ depending on $\gamma_1,\beta_U,d_1,L_f,R_f,b_U$  and $C_{1}$ depending on $\gamma_1,d_1,\beta_U$ such that the following holds:
     For any $\varepsilon>0$,
     set $H={ (C\sqrt{b_U})^{b_U}}\varepsilon^{-b_U}$ and the network $\cF_{\rm NN}(n_x, 1 , L, p, K, \kappa, R)$ 
     with 
     \begin{align*}
        &L = O\left(\log(\varepsilon^{-1})\right),\ p = O(1), \ {K = O\left(\log(\varepsilon^{-1})\right),} \\ 
        &\kappa=O(\varepsilon^{-b_U-1}), \ R=1.
    \end{align*}
    There are 
     \begin{align}
        \sup_{u\in U}\left|f(u)-\sum_{k=1}^{H} a_k\widetilde{q}_k(\ub)\right|\leq \varepsilon,
    \end{align}
    where $\ub=[u(\xb_1), u(\xb_2),...,u(\xb_{n_x})]^{\top}$, $a_k$'s are coefficients depending on $f$ and satisfying $|a_k|\leq C_1$. 
    The constant hidden in $O$ depends on $\gamma_1,\beta_U,d_1,L_f,L_U,b_U,C_A$. 
\end{theorem}
Theorem \ref{thm_functional_lowd} is proved in Section \ref{proof.functional.lowd}. Theorem \ref{thm_functional_lowd} expresses the functional network as a sum of $H$ parallel branches, and the network architecture of each branch is quantified. In the following, we express the functional network as one large network and quantify the network architecture of this large network.

    If we set the network architecture $\cF_{\rm NN}(n_x,1,L, p, K, \kappa, R)$ 
    as
    \begin{align}
        &L = O\left(\log(\varepsilon^{-1})\right),\  p = O(\varepsilon^{-b_U}),\ K = O\left(\varepsilon^{-b_U}\log(\varepsilon^{-1})\right), \ \kappa=O(\varepsilon^{-b_U-1}),\ R=R_f,
        \label{eq:fnetsize}
    \end{align}
    for any Lipschitz functional $f$ with Lipschitz constant no more than $L_f$ and $\|f\|_{L^{\infty}(U)}\leq R_f$,
    there exists $\widetilde{f}\in \cF_{\rm NN}(n_x,1,L, p, K, \kappa, R)$ such that
      we have 
    \begin{align}
        \sup_{u\in U}|f(u)-\widetilde{f}(\ub)|\leq \varepsilon.
    \end{align}
   The constant hidden in $O$ in \eqref{eq:fnetsize} depends on $\gamma_1,\beta_U,d_1,L_f,L_U,L_f,b_U,C_A$.

Now we are ready to prove Theorem \ref{thm_operator_lowd}.

\begin{proof}[Proof of Theorem \ref{thm_operator_lowd}]
    We follow the proof of Theorem \ref{thm_operator} until (\ref{eq.proof.approx.fLip}). 
    By Theorem \ref{thm_functional_lowd}, there exists a network architecture $\cF_2=\cF_{\rm NN} (n_x,1,L_2,p_2,K_2,\kappa_2,R_2)$ with
    \begin{align*}
        &L_2 = O\left(\log(\varepsilon_2^{-1})\right),\ p_2 = O(\varepsilon_2^{-b_U}), \ K_2 = O\left(\varepsilon_2^{-b_U}\log(\varepsilon_2^{-1})+n_x\right), \ \kappa_2=O(\varepsilon_2^{-b_U-1}), \ R_2=1
    \end{align*}
    such that for every functional $f_k$ defined in (\ref{eq:fk}), this network architecture gives a network $\widetilde{f}_k$ satisfying 
    $$
    \sup_{u\in U} |f_k(G(u))-\widetilde{f}_k(\ub)|\leq \varepsilon_2.
    $$
    The constant hidden in $O$ depends on $\gamma_1,\beta_U,\beta_V,d_1,L_G,L_U,C_A,b_U$. Since $0\leq\widetilde{q}_k(\yb)\leq 1$ for any $\yb\in \Omega_V$, we deduce
    
    \begin{align}
        &\sup_{\yb\in \Omega_V}\left|\sum_{k = 1}^{N} f_k(G(u))\widetilde{q}_k(\yb) - \sum_{k = 1}^{N} \tilde{f 
        }_k(\ub ) \widetilde{q}_k(\yb)\right| \nonumber\\
        =& \sup_{\yb\in \Omega_V}\left|\sum_{k = 1}^{N} \left(f_k(G(u))-\widetilde{f}_k(\ub) \right) \widetilde{q}_k(\yb)\right| \nonumber\\
        \leq& \sup_{\yb\in \Omega_V}\sum_{k = 1}^{N} \left|f_k(G(u))-\widetilde{f}_k(\ub)  \right|\widetilde{q}_k(\yb) \nonumber\\
        \leq& \left\|f_k(G(u))-\widetilde{f}_k(\ub )\right\|_{L^{\infty}(U)}\sup_{\yb\in \Omega_V}\sum_{k = 1}^{N} \widetilde{q}_k(\yb) \nonumber\\
        \leq& \varepsilon_2\sup_{\yb\in \Omega_V}\sum_{k = 1}^{N} \widetilde{q}_k(\yb) \nonumber\\
        \leq& (1+1/\gamma_2)\varepsilon_2, 
        \label{eq.operator.lowd.2}
    \end{align}
where the last inequality follows from (\ref{eq.sumq}).
 Putting (\ref{eq.operator.1}) and (\ref{eq.operator.lowd.2}) together, we have
    \begin{align*}
        &\sup_{u\in U} \sup_{\yb\in \Omega_V}\left|G(u)(\yb)-\sum_{k = 1}^{N} \widetilde{f}_k(\ub )\widetilde{q}_k(\yb)\right| \nonumber \\
        \leq & \sup_{u\in U} \sup_{\yb\in \Omega_V}\left|G(u)(y)-\sum_{k = 1}^{N} f_k(G(u))\widetilde{q}_k(y)\right| + \sup_{u\in U} \sup_{\yb\in \Omega_V}\left|\sum_{k = 1}^{N} f_k(G(u))\widetilde{q}_k (y)-\sum_{k = 1}^{N} \widetilde{a}_k(\ub )\widetilde{q}_k(y)\right|\\
        \leq& \varepsilon_1 + (1+1/\gamma_2)\varepsilon_2.
    \end{align*}
    Set {$\varepsilon_2 = \varepsilon/(2(1+1/\gamma_2)),\varepsilon_1=\frac{\varepsilon}{2}$}, we have
    \begin{align*}
        &\sup_{u\in U} \sup_{\yb\in \Omega_V}\left|G(u)(\yb)-\sum_{k=1}^{N} \widetilde{a}_k (\ub)\widetilde{q}_k(\yb)\right|\leq \varepsilon.
    \end{align*}
The resulting network architecture  has $N=O(\varepsilon^{-d_2})$ branch and trunk sub-networks. Each branch sub-network has parameters 
\begin{align*}
        L_1 = O\left(\log(\varepsilon^{-1})\right),\ p_1 = O(1),\  K_1 = O\left(\log(\varepsilon^{-1})\right),\  \kappa_1=O(\varepsilon^{-1}),\ R_1=1,
    \end{align*}
%\hao{for branch net,} and 
and each trunk sub-network has parameters {
\begin{align*}
&L_2 = O\left(\log(\varepsilon^{-1})\right),\  p_2 = O(\varepsilon^{-b_U}),\ K_2 = O\left(\varepsilon^{-b_U}(\log(\varepsilon^{-1})+n_x)\right), \\ &\kappa_2=O(\varepsilon^{-b_U-1}),\ R=\beta_V.
    \end{align*}}
%\hao{for trunk net.}
The constant hidden in $O$ depends on $\gamma_1,\gamma_2,\beta_U,\beta_V,d_1,d_2,L_G,L_U,L_V,C_A,b_U$.
\end{proof}

\subsection{Proof of Theorem \ref{thm_generalization_lowd}}
\label{proof.generalization.lowd}

\begin{proof}[Proof of Theorem \ref{thm_generalization_lowd}]
    We follow the proof of Theorem \ref{thm.generalization} until (\ref{eq.gene.err.1}) and replace Corollary \ref{coro_operator} by Theorem \ref{thm_operator_lowd} and $\cF_2$ by $\cF_2 = \cF_{\rm NN} (n_x,1,L_2,p_2,K_2,\kappa_2,R_2)$ with
    \begin{align}
        &L_2 = O\left(\log(\varepsilon^{-1})\right),\  p_2 = O(\varepsilon^{-b_U}),\ K_2 = O\left(\varepsilon^{-b_U}(\log(\varepsilon^{-1})+n_x)\right), \nonumber\\ &\kappa_2=O(\varepsilon^{-b_U-1}),\ R=\beta_V.
        \label{eq.lowd.gene.trunk.net.1}
    \end{align}
    
    Substituting (\ref{eq.gene.branch.net.1}) and (\ref{eq.lowd.gene.trunk.net.1}) into Lemma \ref{lem.covering} gives rise to 
\begin{align}
    &\log \cN(\theta,\cG_{\rm NN},\|\cdot\|_{\infty,\infty}) \nonumber\\
    = & NK_1(\log 2+\log L_1+2\log p_1+\log \kappa_1+\log H+ \log \frac{1}{\theta}) \nonumber \\
    &+  NK_2(\log 2+\log L_2+2\log p_2+\log \kappa_2+\log H+ \log \frac{1}{\theta}) \nonumber\\
    \leq &C_4N(K_1+K_2)(\log H + \log \frac{1}{\theta})  \nonumber\\
    \leq & C_4N(K_1+K_2)(\log N+L_2(\log L_2+ \log p_2+\log \kappa_2)+\log\frac{1}{\theta} ) \nonumber \\
    \leq& C_4\varepsilon^{-b_U-d_2}\left(\log \frac{1}{\varepsilon}+n_x\right)\left(\log^2 \frac{1}{\varepsilon}+\log \frac{1}{\theta}\right) 
    \label{eq.lowd.gene.covering.bound}
\end{align}

where $C_4$ is a constant depending on $\gamma_1,\gamma_2,\beta_U,\beta_V,d_1,d_2,L_G,L_U,L_V,b_U,C_A$.

Substituting (\ref{eq.lowd.gene.covering.bound}) into (\ref{eq.gene.err.1}) gives rise to{\color{black}
\begin{align}
    &\EE_{\cS}\EE_{u\sim \rho_u} \EE_{\{\yb_j\}_{j=1}^{n_y}\sim \rho_y}\left[\frac{1}{n_y}\sum_{j=1}^{n_y}(\widehat{G}(\ub)(\yb_j) - G(u)(\yb_j))^2\right] \nonumber\\
	\leq& 4\varepsilon^2 + \left(\frac{64\sigma^2+6}{nn_y}+\frac{19\beta_V^2}{n} \right)C_4{\varepsilon^{-b_U-d_2}\left(\log \frac{1}{\varepsilon}+n_x\right)\left(\log^2 \frac{1}{\varepsilon}+\log \frac{1}{\theta}\right)} \nonumber\\
	&+16\sigma\theta\sqrt{\frac{4 C_4{ \varepsilon^{-b_U-d_2}\left(\log \frac{1}{\varepsilon}+n_x\right)\left(\log^2 \frac{1}{\varepsilon}+\log \frac{1}{\theta}\right)} + 6}{nn_y} }+ (8\sigma+6)\theta.
\end{align}
Set $\theta=n^{-1}$ and $\varepsilon=n^{-\frac{1}{{ 2+b_U+d_2}}}$, we have
\begin{align}
    \EE_{\cS}\EE_{u\sim \rho_u} \EE_{\{\yb_j\}_{j=1}^{n_y}\sim \rho_y}\left[\frac{1}{n_y}\sum_{j=1}^{n_y}(\widehat{G}(\ub)(\yb_j) - G(\ub)(\yb_j) )^2\right]\leq C_2n^{-\frac{2}{{ 2+b_U+d_2}}}\log^3\frac{1}{\varepsilon}
\end{align}
for some $C_2$ depending on $\sigma,\gamma_1,\gamma_2,\beta_U,\beta_V,d_1,d_2,L_G,L_U,L_V,b_U,C_A$.}
\end{proof}

\section{Conclusion}
\label{sec.conclusion}
In this paper, we have developed mathematical and statistical theories to justify  neural scaling laws of DeepONet by analyzing its approximation and generalization error.  
Our approximation theory can be used to quantify the model scaling law of DeepONet, depicting the scaling between the DeepONet approximation error and the model size. Our generalization theory can be used to quantify the data scaling law of DeepONet, depicting the scaling between  the DeepONet generalization error and the training data size. Our general results for learning Lipschitz operators give rise to a slow rate of convergence of the DeepONet error as the model/data size increases. Furthermore, we incorporate low-dimensional structures of the input functions into consideration, and improve the rate of convergence to a power law, which is consistent with the empirical observations in  \citet{lu2021learning,de2022cost}. 
Our results provide theoretical foundations of DeepONet, to partially explain the empirical success and  neural scaling laws of DeepONet. In the future, we will investigate the optimality of our error bound, and improve it if possible. Another interesting future direction is to incorporate more complicated low-dimensional structures, {\color{black}such as manifold setting,} under the framework of DeepONet.

\bibliographystyle{ims}
\bibliography{ref}

\newpage
\appendix
\section*{Appendix}

\section{Proof of Examples in Section \ref{sec.setup}}
\subsection{Proof of Example \ref{ex.fourier}}
\label{proof.ex.fourier}
    \begin{proof}[Proof of Example \ref{ex.fourier}]

    The exact solution is given as 
    $$
    v(\xb,T)=u(\xb+T\cbb).
    $$
     For any $u,\bar{u}\in U$ with $u=\sum_{j=1}^J a_jw_j, \bar{u}=\sum_{j=1}^J \bar{a}_jw_j$, we have
    \begin{align*}
        \|G(u)-G(\bar{u})\|_{L^{\infty}(\Omega_V)}=& \left\| \sum_{j=1}^J |a_j-\bar{a}_j|w_j(\xb+T\cbb)\right\|_{L^{\infty}(\Omega_V)}\leq \|\ab-\bar{\ab}\|_{\ell^1}\\
        \leq& \sqrt{J} \|\ab-\bar{\ab}\|_{\ell^2} =\sqrt{J}\|u-\bar{u}\|_{L^2(\Omega_U)}.
    \end{align*}

    \end{proof}

\section{Proof of Theorems and Lemmata in Section  \ref{sec.main_proof}}

\subsection{Proof of Theorem \ref{thm_function}}
\label{proof.function}

\begin{proof}[Proof of Theorem \ref{thm_function}]
We partition $\Omega_U$ into $N^{d_1}$ subcubes for some $N$ to be specified later. We are going to approximate $u$ on each cube by a constant function and then assemble them together to get an approximation of $u$ on $\Omega_U$. Denote the centers of the subcubes by $\{\cbb_k\}_{k=1}^{N^{d_1}}$ with $\cbb_k = [c_{k, 1}, c_{k, 2}, ..., c_{k, d_1}]^{\intercal}$.

Let $\{\cbb_k\}_{k=1}^{N^{d_1}}$ be a uniform grid on $\Omega_U$ so that each $\cbb_k\in \left\{-\gamma_1,-\gamma_1+\frac{2\gamma_1}{N-1},..., \gamma_1\right\}^{d_1}$ for each $k$. 
Define 
\begin{align}
    \psi(a) = \begin{cases}
        1, |a|<1,\\
        0, |a|>2, \\
        2-|a|, 1\leq |a|\leq 2,
    \end{cases}
    \label{eqn_psi}
\end{align}
with $a\in\mathbb{R}$, and 
\begin{align}
    \phi_{\cbb_k}(\xb) = \prod_{j = 1}^{d_1} \psi \left(\frac{3(N-1)}{2\gamma_1}(x_j-c_{k,j})\right).
    \label{eqn_phi}
\end{align}
In this definition, we have $\supp(\phi_{\cbb_k})=\left\{\xb: \|\xb-\cbb_k\|_{\infty}\leq \frac{4\gamma_1}{3(N-1)}\right\}\subset \left\{\xb: \|\xb-\cbb_k\|_{\infty}\leq \frac{2\gamma_1}{(N-1)}\right\}$ and 
$$
\|\phi_{\cbb_k}\|_{L^{\infty}(\Omega_U)}=1, \quad \sum_{k=1}^{N^{d_1}} \phi_{\cbb_k}=1.
$$
For any $u\in U$, we construct a piecewise constant approximation to $u$ as
$$
\bar{u}(\xb)=\sum_{k=1}^{N^{d_1}} u(\cbb_k)\phi_{\cbb_k}(\xb).
$$
By utilizing the partition of unity property given by $\sum_{k=1}^{N^{d_1}} \phi_{\cbb_k}=1$, it follows that for any $\xb\in \Omega_U$,
\begin{align}
    |u(\xb)-\bar{u}(\xb)|=&\left| \sum_{k=1}^{N^{d_1}} \phi_{\cbb_k}(\xb)(u(\xb)-u(\cbb_k))\right| \nonumber\\
    \leq & \sum_{k=1}^{N^{d_1}} \phi_{\cbb_k}(\xb)|u(\xb)-u(\cbb_k)| \nonumber\\
    = &\sum_{k: \|\cbb_k-\xb\|_{\infty} \leq \frac{2\gamma_1}{(N-1)}}  \phi_{\cbb_k}(\xb)|(u(\xb)-u(\cbb_k))| \nonumber\\
    \leq & \max_{k: \|\cbb_k-\xb\|_{\infty}\leq \frac{2\gamma_1}{(N-1)}}|(u(\xb)-u(\cbb_k))|\left(\sum_{k: \|\cbb_k-\xb\|_{\infty} \leq \frac{2\gamma_1}{(N-1)}}  \phi_{\cbb_k}(\xb)\right) \nonumber\\
    \leq& \max_{k: \|\cbb_k-\xb\|_{\infty}\leq \frac{2\gamma_1}{(N-1)}}|(u(\xb)-u(\cbb_k))| \nonumber\\
    \leq & \frac{2\sqrt{d_1}\gamma_1L_U}{N-1}, \label{eqn_approx_by_ubar}
\end{align}
where we use the Lipschitz assumption of $U$ in the last inequality.
Setting $N=\left\lceil \frac{4\sqrt{d_1}\gamma_1L_U}{\varepsilon}\right\rceil+1$ gives rise to
\begin{align}
        |u(\xb)-\bar{u}(\xb)|\leq \frac{\varepsilon}{2}, \quad \forall \xb\in \Omega_U.
        \label{eq.function.taylor}
\end{align}
We then show that $\phi_{\cbb_k}$ can be approximated by a network with arbitrary accuracy. Note that $\phi_{\cbb_k}$ is the product of $d_1$ functions, each of which is piecewise linear and can be realized by 4-layer ReLU networks. 

The following lemma shows that a function of the product can be approximated by a network with arbitrary accuracy.
\begin{lemma}[Proposition 3 of \citet{yarotsky2017error}]
\label{lemma_pp3}
    Given $M>0$ and $\varepsilon>0$, there is a ReLU network $\widetilde{\times}: \mathbb{R}^2\rightarrow \mathbb{R}$ in $\cF_{\rm NN}(2,1,L,p,K,\kappa,R)$ such that for any $|x|\leq M, |y|\leq M$, we have
   $$|\widetilde{\times}(x,y) -xy|<\varepsilon.$$
The network architecture has 
\begin{align}
    L=O(\log \varepsilon^{-1}),\ p=6,\ K=O(\log \varepsilon^{-1}),  \ \kappa=O(\varepsilon^{-1}), \ R=M^2.
\end{align}
The constant hidden in $O$ depends on $M$.
\end{lemma}

Let $\widetilde{\times}$ be the network defined in Lemma \ref{lemma_pp3} with accuracy $\delta$. We approximate $\phi_{\cbb_k}$ by $\bar{q}_k$ defined as, 
 \begin{align*}
        \bar{q}_{k}(\xb)= \widetilde{\times}\left(\psi\left(\frac{3(N-1)}{2\gamma_1}(x_1-c_{k,1})\right), \widetilde{\times}\left(\psi\left(\frac{3(N-1)}{2\gamma_1}(x_2-c_{k,2})\right), \cdot\cdot\cdot \right)\right).
    \end{align*}
    
For each $k$, $\bar{q}_k\in \cF_{\rm NN}(d_1, 1, L, p, K, \kappa, R)$ with
\begin{align*}
    L = O(d_1\log \delta^{-1}), p = O(1), K = O(d_1\log \delta^{-1}), \kappa = O(\delta^{-1}+N), R = 1.
\end{align*}
For any $\xb\in \Omega_U$, we have 
    \begin{align*}
        &|\bar{q}_{k}(\xb)-\phi_{\cbb_{k}}(\xb) |\\
        \leq &\left|\widetilde{\times}\left(\psi\left(\frac{3(N-1)}{2\gamma_1}(x_1-c_{k,1})\right), \widetilde{\times}\left(\psi\left(\frac{3(N-1)}{2\gamma_1}(x_2-c_{k,2})\right), \cdot\cdot\cdot \right)\right) - \phi_{\cbb_{k}}(\xb) \right|\\
        \leq & \bigg|\widetilde{\times}\left(\psi\left(\frac{3(N-1)}{2\gamma_1}(x_1-c_{k,1})\right), \widetilde{\times}\left(\psi\left(\frac{3(N-1)}{2\gamma_1}(x_2-c_{k,2})\right), \cdot\cdot\cdot \right)\right) \\
        & - \psi\left(\frac{3(N-1)}{2\gamma_1}(x_1-c_{k,1})\right)\widetilde{\times}\left(\psi\left(\frac{3(N-1)}{2\gamma_1}(x_2-c_{k,2})\right), \cdot\cdot\cdot \right) \bigg|\\
        &+ \bigg| \psi\left(\frac{3(N-1)}{2\gamma_1}(x_1-c_{k,1})\right)\widetilde{\times}\left(\psi\left(\frac{3(N-1)}{2\gamma_1}(x_2-c_{k,2})\right), \cdot\cdot\cdot \right)  -\phi_{\cbb_{k}}(\xb) \bigg|\\
        \leq & \delta + \mathcal{E}_2,
    \end{align*}
    where
    \begin{align*}
        \mathcal{E}_2 &= \bigg|\psi\left(\frac{3(N-1)}{2\gamma_1}(x_1-c_{k,1})\right)\widetilde{\times}\left(\psi\left(\frac{3(N-1)}{2\gamma_1}(x_2-c_{k,2})\right), \cdot\cdot\cdot \right)  -\phi_{\cbb_{k}}(\xb) \bigg|\\
        & = \bigg| \psi\left(\frac{3(N-1)}{2\gamma_1}(x_1-c_{k,1})\right) \bigg|
       \bigg| \widetilde{\times}\left(\psi\left(\frac{3(N-1)}{2\gamma_1}(x_2-c_{k,2})\right), \cdot\cdot\cdot \right) - \prod_{j = 2}^{d_1} \psi \left(\frac{3(N-1)}{2\gamma_1}(x_j-c_{k,j})\right) \bigg|
    \end{align*}
Repeat this process to estimate $\mathcal{E}_2, \mathcal{E}_3, ..., \mathcal{E}_{d_1+1}$, where $\mathcal{E}_{d_1+1} = \prod\limits_{k = 1}^{d_1}\psi\left(\frac{3(N-1)}{2\gamma_1}(x_2-c_{k,2})\right) - \phi_{\cbb_k} = 0$. This implies that $\|\phi_{\cbb_{k}} - \bar{q}_{k}\|_{L^{\infty}(\Omega_U)}\leq d_1\delta$. {We define $$\widetilde{q}_k=\ReLU(\bar{q}_k)=\max\{0,\bar{q}_k\}.$$
Since $\phi_{\cbb_{k}}\geq0$, we have 
$$
\|\phi_{\cbb_{k}} - \widetilde{q}_{k}\|_{L^{\infty}(\Omega_U)}\leq \|\phi_{\cbb_{k}} - \bar{q}_{k}\|_{L^{\infty}(\Omega_U)}\leq d_1\delta
$$
and $\widetilde{q}_k\geq 0$.}
It follows that,
    \begin{align}
        \left\| \sum_{k=1}^{N^{d_1}} u(\cbb_k)\widetilde{q}_k - \bar{u} \right\|_{L^{\infty}(\Omega_U)} 
        = & \left\| \sum_{k=1}^{N^{d_1}} u(\cbb_k)\widetilde{q}_k - \sum_{k=1}^{N^{d_1}} u(\cbb_k)\phi_{\cbb_k} \right\|_{L^{\infty}(\Omega_U)} \nonumber \\
        \leq & \sum_{k=1}^{N^{d_1}}|u(\cbb_k)|\|\widetilde{q}_k-\phi_{\cbb_k}\|_{L^{\infty}(\Omega_U)}\nonumber \\
        \leq &  d_1N^{d_1}\beta_U\delta. 
        \label{eq.function.net}
    \end{align}
Setting $\delta = \frac{\varepsilon}{2d_1N^{d_1}\beta_U}$ and putting (\ref{eq.function.taylor}) and (\ref{eq.function.net}) together, we have 
\begin{align}
    \left\| u-\sum_{k=1}^{N^{d_1}} u(\cbb_k)\widetilde{q}_k\right\|_{L^{\infty}(\Omega_U)}\leq \left\| u-\bar{u}\right\|_{L^{\infty}(\Omega_U)} + \left\| \bar{u}-\sum_{k=1}^{N^{d_1}} u(\cbb_k)\widetilde{q}_k\right\|_{L^{\infty}(\Omega_U)}\leq \frac{\varepsilon}{2}+ \frac{\varepsilon}{2}=\varepsilon.
\end{align}

The network architecture is specified in the theorem. 
\end{proof}

\subsection{Proof of Theorem \ref{thm_functional}}
\label{proof.functional}

\begin{proof}[Proof of Theorem \ref{thm_functional}]

We let $\{\mathcal{B}_{\delta}(\cbb_m) \}_{ m  = 1}^{c_U}$ be a finite cover of $\Omega_U$ by $c_U$ Euclidean balls, where $c_U$ can be further estimated in Corollary \ref{coro_functional}.
By Lemma \ref{lemma_pou}, there exists a partition of unity 
$\{\omega_m(\xb)\}_{m=1}^{c_U}$ subordinate to the cover $\{\mathcal{B}_{\delta}(\cbb_m)\}_{m=1}^{c_U}$. 
For any $\zb = [z_1, ..., z_{c_U}]^\top\in (-\beta_U, \beta_U)^{c_U}$, we can then define a function $z_{\omega}(\xb):\Omega_U \rightarrow \RR$ such that 
\begin{align}
        z_{\omega}(\xb) = \sum_{m=1}^{c_U} \zb_m\omega_m(\xb) \quad \forall \xb\in \Omega_U.
        \label{eqn_z_rho}
    \end{align}
    Note that for $u\in U$ and $\ub =[u(\cbb_1),...,u(\cbb_{c_U})]^{\top}$, if we set 
    $\zb=\ub$, it follows that
    $u_{\omega}=z_{\omega}\in \bar{U}_{\beta_U}$ is an approximation of $u$ with the  point-wise error estimation: 
    \begin{align*}
        |u(\xb)-u_{\omega}(\xb)|\leq &\sum_{m=1}^{c_U} |u(\xb) - u(\cbb_m)| |\omega_m(\xb)|\\
        =&\sum_{m: \|\xb - \cbb_m\|_2\leq \delta} |u(\xb)-u(\cbb_m)||\omega_m(\xb)| \textcolor{black}{\leq  L_U\delta}
    \end{align*}
    for any $\xb \in \Omega_U$.
Setting $\delta=\frac{\varepsilon}{2(2\gamma_1)^{d_1/2}L_fL_U}$ and using the Lipschtiz property of $f$, we have 
    \begin{align*}
       |f(u)-f(u_{\omega})|\leq L_f\|u-u_{\omega}\|_{L^2(\Omega_U)}\leq L_f(2\gamma_1)^{d_1/2}L_U\delta=\frac{\varepsilon}{2}.
    \end{align*}
    
   We next define a function $g: (-\beta_U, \beta_U)^{c_U}\rightarrow \mathbb{R}$ such that $g(\zb) = f(z_{\omega})$, i,e., $g(\zb) = f(u_{\omega})  $.
   We claim that $g$ is Lipschitz in the following sense:
    For any $u, \bar{u}\in U$, define $u_{\omega}$ and $\bar{u}_\omega$ as in  \eqref{eqn_z_rho} where $\ub =[u(\cbb_1),...,u(\cbb_{c_U})]^{\top}$ and $\bar{\ub} =[\bar u(\cbb_1),...,\bar u(\cbb_{c_U})]^{\top}$. Then we have
    \begin{align*}
       |g(\ub) - g(\bar{\ub})| = & |f(u_{\omega})-f(\bar{u}_{\omega})|\\
       \leq &L_f\|u_{\omega} - \bar{u}_{\omega}\|_{L^2(\Omega_U)}\\
       =& L_f \sqrt{\int_{\Omega_U} (u_{\omega} - \bar{u}_{\omega})^2 d\xb}\\
    =& L_f \sqrt{\int_{\Omega_U} \left( \sum_{m=1}^{c_U} \left(u(\cbb_m)-\bar{u}(\cbb_m)\right)\omega_m(\xb)\right)^2 d\xb}\\
\leq& L_f \sqrt{\int_{\Omega_U} \sum_{m=1}^{c_U} \left(u(\cbb_m)-\bar{u}(\cbb_m)\right)^2 \ \sum_{m=1}^{c_U} \left(\omega_m(\xb)\right)^2 \ d\xb}\\
\leq& L_f \sqrt{\int_{\Omega_U} \sum_{m=1}^{c_U} \left(u(\cbb_m)-\bar{u}(\cbb_m)\right)^2 \ \sum_{m=1}^{c_U} \omega_m(\xb) \ d\xb}\\
\leq& L_f \sqrt{\int_{\Omega_U} \sum_{m=1}^{c_U} \left(u(\cbb_m)-\bar{u}(\cbb_m)\right)^2  \ d\xb}\\
=& L_f \, |\Omega_U|^{\frac{1}{2}} \, \| \ub - \bar{\ub} \|_{2} \\
=& L_f(2\gamma_1)^{d_1/2}\| \ub - \bar{\ub} \|_{2}
    \end{align*}
    where the third equality follows from the property that $\{\omega_m\}_{m=1}^{c_U}$ is a partition of unity.
    The claim is proved. 
    
    By Theorem \ref{thm_function}, for $\varepsilon>0$, if we set {$H=(C\sqrt{c_U})^{c_U}\varepsilon^{-c_U}$} for some $C$ depending on $d_1,\gamma_1,\beta_U$ and $L_f$,  then there exists a 
    network architecture $\cF_{\rm NN}(c_U,1,L,p,K,\kappa,R)$ and $\{\widetilde{q}_k\}_{k=1}^{ H}$ with $\widetilde{q}_k\in \cF_{\rm NN}(c_U,1,L,p,K,\kappa,R)$ for $k=1,\ldots,H$ such that 
    \begin{align*}
        \sup_{u\in U}\left|g(\ub )-\sum_{k=1}^{H} a_k \widetilde{q}_k(\ub )\right|\leq \frac{\varepsilon}{2},
    \end{align*}
    where $a_k$ are constants depending on $f$ with $|a_k|\leq R_f$. 
    Such an architecture has
    \begin{align*}
        &L = O\left({c_U^2\log(c_U)} +{ c_U^2}\log(\varepsilon^{-1})\right),\  p = O(1),\ K = O\left(c_U^2\log c_U+c_U^2\log(\varepsilon^{-1})\right), \\ &\kappa=O(c_U^{c_U/2+1}\varepsilon^{-c_U-1}),\ R=1.
    \end{align*}
{The constant hidden in $O$ depends on $d_1,\gamma_1,\beta_U$ and $L_f$.}
    We have, for any $u  \in U$ and $\ub =[u(\cbb_1),...,u(\cbb_{c_U})]^{\top}$
    \begin{align*}
        \sup_{u \in U} \left|f(u)-\sum_{k=1}^{H} a_k \widetilde{q}_k(\ub )\right| \leq &\sup_{u\in U} \left|f(u)-g(\ub) \right| + \sup_{\bu}\left| g(\ub)-\sum_{k=1}^{H} a_k \widetilde{q}_k(\ub )\right|\\
        \leq &\frac{\varepsilon}{2}+ \frac{\varepsilon}{2}=\varepsilon.
    \end{align*}
\end{proof}

\subsection{Proof of Corollary \ref{coro_functional}}
\label{proof.coro_functional}

\begin{proof}[Proof of Corollary \ref{coro_functional}]
  $\Omega_U$ is bounded and closed; hence it is compact. Let $\{\mathcal{B}_{\delta}(\cbb_m) \}_{ m  = 1}^{c_U}$ be a finite cover of $\Omega_U$ by $c_U$ Euclidean balls, with centers $\{\cbb_m\}_{m=1}^{c_U}$ and radius $\delta$.
By Lemma \ref{lem.cover.ball}, we have,
    \begin{align}
        \textcolor{black}{c_U\leq C_2\delta^{-d_1}}=C_2\left(\frac{2(2\gamma_1)^{d_1/2}L_fL_U}{\varepsilon}\right)^{d_1}
        \label{eqn_chapter2_col2}
    \end{align}
    for some $C_2$ depending on $\gamma_1$ and $d_1$. Then Corollary \ref{coro_functional} is a direct result of Theorem \ref{thm_functional}.
\end{proof}

\subsection{Proof of Theorem \ref{thm_functional_lowd}}
\label{proof.functional.lowd}

\begin{proof}[Proof of Theorem \ref{thm_functional_lowd}]
    Under Assumption \ref{assumption_V} and Assumption \ref{assumption_lowd}(ii), for each $u\in U$, we have $|\alpha_k|\leq C_{\alpha}$ for $k=1,\ldots,b_U$ where $C_{\alpha}=(2\gamma_1)^{d_1/2}\beta_U$.
    For any $\zb = [z_1, ..., z_{b_U}]^{\top}\in [-C_{\alpha},C_{\alpha}]^{b_U}$,  we define the function $z_{\omega}:\Omega_U \rightarrow \RR$ such that 
    \begin{align}
        z_{\omega}(\xb)=\sum_{m=1}^{b_U} \zb_m \omega_m (\xb),\ \forall \xb\in \Omega_U,
    \end{align}
    where $\{\omega_m\}_{m=1}^{b_U}$ are the orthonormal basis  in Assumption \ref{assumption_lowd}(i).
    
 For  $u\in U$ and $\ub=[u(\xb_1),...,u(\xb_{n_x})]^{\top}$, if $\zb=A\ub$ then we have $u=z_{\omega}$.
    Let us define the function $g:[-C_{\alpha},C_{\alpha}]^{b_U}\rightarrow \RR$ such that $g(\zb)=f(z_{\omega})$ (i.e., $g(A\ub)=f(u)$). 
    Then $g$ is Lipschitz in the following sense: For any $u,\bar{u}\in U$, let $\ub=[u(\xb_1),...,u(\xb_{n_x})]^{\top}$, $\bar{\ub}=[\bar{u}(\xb_1),...,\bar{u}(\xb_{n_x})]^{\top}$, and then we have
    \begin{align*}
        |g(\zb)-g(\bar{\zb})|=&|f(z_{\omega})-f(\bar{z}_{\omega})|\\
        \leq& L_f\|z_{\omega}-\bar{z}_{\omega}\|_{L^2(\Omega_U)}\\
        = &L_f\sqrt{\int_{\Omega_U} (z_{\omega}-\bar{z}_{\omega})^2 d\xb}\\
        \leq & L_f  \sqrt{\int_{\Omega_U} \left(\sum_{m=1}^{b_U} |\zb_m-\bar{\zb}_m|\omega_m(\xb)\right)^2d\xb}\\
        \leq & L_f  \sqrt{\int_{\Omega_U} \sum_{m=1}^{b_U} |\zb_m-\bar{\zb}_m|^2\omega_m^2(\xb)d\xb}\\
        = & L_f  \sqrt{\sum_{m=1}^{b_U} |\zb_m-\bar{\zb}_m|^2\int_{\Omega_U} \omega_m^2(\xb)d\xb}\\
        = &L_f \|\zb-\bar{\zb}\|_2.
    \end{align*}
    By Theorem \ref{thm_function}, for $\varepsilon>0$, set $H={(C\sqrt{b_U})^{b_U}}\varepsilon^{-b_U}$ for some $C$ depending on $b_U,d_1,\gamma_1,\beta_U$ and $L_f$. 
    There exists a network architecture $\cF_{\rm NN}(n_x,1,L,p,K,\kappa,R)$ and $\{\widetilde{q}_k\}_{k=1}^H$ with $\widetilde{q}_k\in \cF_{\rm NN} (n_x,1,L,p,K,\kappa,R)$ and $\widetilde{q}_k\geq0$,  for $k=1,...,H$ such that
    \begin{align}
        \sup_{u\in U} \left| g(A\ub)-\sum_{k=1}^H a_k\widetilde{q}_k(\ub)\right| \leq \varepsilon,
    \end{align}
    where $a_k$ are constants depending on $f$ with $|a_k|\leq R_f$. Such an architecture has
    \begin{align*}
        L=O(\log(\varepsilon^{-1})),\ p=O(1), \ K=O(\log(\varepsilon^{-1})+n_x), \ \kappa=O(\varepsilon^{-b_U-1}), \ R=1.
    \end{align*}
    Note that $\kappa$ depends on $C_A$ as defined in Assumption \ref{assumption_lowd} as the network weights are scaled up by $A$.
    We have for any $u\in U$,
    \begin{align*}
        \sup_{u\in U} \left| f(u)-\sum_{k=1}^H a_k\widetilde{q}_k(\ub)\right| = \sup_{u\in U} \left| g(A\ub)-\sum_{k=1}^H a_k\widetilde{q}_k(\ub)\right|=\varepsilon.
    \end{align*}
\end{proof}

\subsection{Proof of Lemma \ref{lem.cover.ball}}
\label{proof.lem.cover.ball}

\begin{proof}[Proof of Lemma \ref{lem.cover.ball}]
     By \citet[Chapter 2]{conway2013sphere}, we have,
    \begin{align}
        \textcolor{black}{c\leq \left\lceil\frac{2\gamma}{\delta}\right\rceil^{d}+7d\log d\leq C\delta^{-d}}
        \label{eqn_chapter2_proof}
    \end{align}
    for some $C$ depending on $\gamma$ and $d$.
 \end{proof}

\subsection{Proof of Lemma \ref{lem.gene.T1.2}}
\label{proof.lem.gene.T1.2}
\begin{proof}[Proof of Lemma \ref{lem.gene.T1.2}]
We derive an upper bound of the second term in (\ref{eq.gene.T1}) using the covering number of $\cG_{\rm NN}$. Denote $\|G_{\rm NN}\|_n^2=\frac{1}{n}\sum_{i=1}^n \frac{1}{n_y}\sum_{j=1}^{n_y}|G_{\rm NN}(\ub_i)(\yb_{i,j})|^2$.
Let $\cG^*=\{G_{k}^*\}_{k=1}^{\cN(\theta,\cG_{\rm NN},\|\cdot\|_{\infty,\infty})}$ be a $\theta$ cover of $\cG_{\rm NN}$, where $\cN(\theta,\cG_{\rm NN},\|\cdot\|_{\infty,\infty})$ is the covering number.
Specifically, for any $G_{\rm NN}\in \cG_{\rm NN}$, there exists $G_{\rm NN}^*\in \cG^*$ satisfying $\|G_{\rm NN}^*-\widehat{G}\|_{\infty,\infty}\leq \theta$. 
We have
    \begin{align}
    &\EE_{\cS}\left[\frac{1}{n}\sum_{i=1}^n \frac{1}{n_y}\sum_{j=1}^{n_y}\widehat{G}(\ub_i)(\yb_{i,j})\xi_{i,j}\right] \nonumber\\
    =&\EE_{\cS}\left[\frac{1}{n}\sum_{i=1}^n \frac{1}{n_y}\sum_{j=1}^{n_y}\left(\widehat{G}(\ub_i)(\yb_{i,j}) - G_{\rm NN}^*(\ub_i)(\yb_{i,j})+G_{\rm NN}^*(\ub_i)(\yb_{i,j})-G(u_i)(\yb_{i,j})\right)\xi_{i,j}\right] \nonumber\\
    =& \EE_{\cS}\left[\frac{1}{n}\sum_{i=1}^n \frac{1}{n_y}\sum_{j=1}^{n_y}\left(\widehat{G}(\ub_i)(\yb_{i,j})-G_{\rm NN}^*(\ub_i)(\yb_{i,j})\right)\xi_{i,j}\right] \nonumber\\
     &+ \EE_{\cS}\left[\frac{1}{n}\sum_{i=1}^n \frac{1}{n_y}\sum_{j=1}^{n_y}\left(G_{\rm NN}^*(\ub_i)(\yb_{i,j})-G(u_i)(\yb_{i,j})\right)\xi_{i,j}\right] \nonumber\\
    	\leq&  \sigma\theta+
     \EE_{\cS}\left[\frac{\|G_{\rm NN}^*-G\|_n}{\sqrt{nn_y}}\frac{\sum_{i=1}^n \sum_{j=1}^{n_y}\left(G_{\rm NN}^*(\ub_i)(\yb_{i,j})-G(u_i)(\yb_{i,j})\right)\xi_{i,j}}{\sqrt{nn_y}\|G_{\rm NN}^*-G\|_n}\right] 
    \label{eq.gene.T1.2.1}.
    \end{align}
Note that
\begin{align}
	&\|G_{\rm NN}^*-G\|_n \nonumber\\
	=&\sqrt{\frac{1}{n}\sum_{i=1}^n \frac{1}{n_y}\sum_{j=1}^{n_y}|G_{\rm NN}^*(\ub_i)(\yb_{i,j})-\widehat{G}(\ub_i)(\yb_{i,j})+\widehat{G}(\ub_i)(\yb_{i,j})-G(u_i)(\yb_{i,j})|^2} \nonumber\\
	\leq& \sqrt{\frac{2}{n}\sum_{i=1}^n \frac{1}{n_y}\sum_{j=1}^{n_y}\left(|G_{\rm NN}^*(\ub_i)(\yb_{i,j})-\widehat{G}(\ub_i)(\yb_{i,j})|^2+|\widehat{G}(\ub_i)(\yb_{i,j})-G(u_i)(\yb_{i,j})|^2\right)} \nonumber\\
	\leq&  \sqrt{\frac{2}{n}\sum_{i=1}^n \frac{1}{n_y}\sum_{j=1}^{n_y}\left(\theta^2+|\widehat{G}(\ub_i)(\yb_{i,j}) - G(u_i)(\yb_{i,j})|^2\right)} \nonumber\\
	=& \sqrt{2}(\|\widehat{G} - G\|_n+\theta).
	\label{eq.gene.T1.2.n}
\end{align}
Substituting (\ref{eq.gene.T1.2.n}) into (\ref{eq.gene.T1.2.1}) gives rise to
\begin{align}
	&\EE_{\cS}\left[\frac{1}{n}\sum_{i=1}^n \frac{1}{n_y}\sum_{j=1}^{n_y}\widehat{G}(\ub_i)(\yb_{i,j})\xi_{i,j}\right] \nonumber\\
	\leq & \sqrt{2} \EE_{\cS}\left[\frac{\|\widehat{G}-G\|_n +\theta}{\sqrt{nn_y}}\left|\frac{\sum_{i=1}^n \sum_{j=1}^{n_y}\left(G_{\rm NN}^*(\ub_i)(\yb_{i,j})-G(u_i)(\yb_{i,j})\right)\xi_{i,j}}{\sqrt{nn_y}\|G_{\rm NN}^*-G\|_n}\right|\right] + \sigma\theta.
 \label{eq.gene.T1.2.2}
\end{align}
Recall that $\{G_{k}^*\}_{k=1}^{\cN(\theta,\cG_{NN},\|\cdot\|_{\infty,\infty})}$ is a $\theta$ cover of $\cG_{NN}$, and denote 
$$
z_k=\frac{\sum_{i=1}^n \sum_{j=1}^{n_y}\left(G_{k}^*(\ub_i)(\yb_{i,j})-G(u_i)(\yb_{i,j})\right)\xi_{i,j}}{\sqrt{nn_y}\|G_{\rm NN}^*-G\|_n}.
$$
We have 
\begin{align}
	 &\EE_{\cS}\left[\frac{\|\widehat{G}-G\|_n +\theta}{\sqrt{nn_y}}\left|\frac{\sum_{i=1}^n \sum_{j=1}^{n_y}\left(G_{\rm NN}^*(\ub_i)(\yb_{i,j})-G(u_i)(\yb_{i,j})\right)\xi_{i,j}}{\sqrt{nn_y}\|G_{\rm NN}^*-G\|_n}\right|\right] \nonumber\\
	 \leq &  \EE_{\cS}\left[\frac{\|\widehat{G}-G\|_n +\theta}{\sqrt{nn_y}}\max_k \left|z_k\right|\right] \nonumber\\
	 \leq& \sqrt{\EE_{\cS}\left[\left(\|\widehat{G}-G\|_n +\theta\right)^2\right] \EE_{\cS}\left[\frac{1}{nn_y}\max_k \left|z_k\right|^2\right]} \nonumber\\
	 \leq& \sqrt{2\EE_{\cS}\left[\|\widehat{G}-G\|_n^2 +\theta^2\right]} \sqrt{\frac{1}{nn_y} \EE_{\cS}\left[\max_k \left|z_k\right|^2\right]} \nonumber\\
    \leq & \sqrt{2}\left(\sqrt{\EE_{\cS}\left[\|\widehat{G}-G\|_n^2\right] } +\theta \right)\sqrt{\frac{1}{nn_y} \EE_{\cS}\left[\max_k \left|z_k\right|^2\right]},
    \label{eq.gene.T1.2.3}
\end{align}
where Cauchy-Schwarz inequality is used in the second inequality, and the last inequality uses the relation $\sqrt{a+b^2}\leq \sqrt{a}+b$ for $a,b\geq 0$.

We next derive an upper bound of $\EE_{\cS}\left[\max_k \left|z_k\right|^2\right]$. Since each $\xi_{i,j}$ is a sub-Gaussian variable with variance proxy $\sigma$, for given $\left\{u_i,\{\yb_{i,j}\}_{j=1}^{n_y}\right\}_{i=1}^n$, each $z_k$ is a sub-Gaussian variable with variance proxy $\sigma^2$. 
Let $t$ be a positive number depending on $\sigma$ and will be made clear later, 
we deduce,
\begin{align*}
    &\EE_{\cS}\left[\left. \max_k \left|z_k\right|^2\right|\left\{u_i,\{\yb_{i,j}\}_{j=1}^{n_y}\right\}_{i=1}^n\right] \nonumber\\
    =& \frac{1}{t} \log \exp \left( \EE_{\cS}\left[\left. t\max_k \left|z_k\right|^2\right|\left\{u_i,\{\yb_{i,j}\}_{j=1}^{n_y}\right\}_{i=1}^n\right] \right) \nonumber\\
    \leq & \frac{1}{t} \log  \EE_{\cS}\left[\exp \left(\left. t\max_k \left|z_k\right|^2\right|\left\{u_i,\{\yb_{i,j}\}_{j=1}^{n_y}\right\}_{i=1}^n\right) \right] \nonumber\\
    \leq & \frac{1}{t} \log  \EE_{\cS}\left[\sum_k \exp \left(\left. t \left|z_k\right|^2\right|\left\{u_i,\{\yb_{i,j}\}_{j=1}^{n_y}\right\}_{i=1}^n\right) \right] \nonumber\\
    =& \frac{1}{t} \log  \left(\cN(\theta,\cG_{NN},\|\cdot\|_{\infty,\infty}) \EE_{\cS}\left[\exp \left(\left. t \left|z_k\right|^2\right|\left\{u_i,\{\yb_{i,j}\}_{j=1}^{n_y}\right\}_{i=1}^n\right) \right]\right) \nonumber\\
    \leq& \frac{1}{t} \log \cN(\theta,\cG_{NN},\|\cdot\|_{\infty,\infty})+ \frac{1}{t} \log \EE_{\cS}\left[ \exp \left(\left. t \left|z_1\right|^2\right|\left\{u_i,\{\yb_{i,j}\}_{j=1}^{n_y}\right\}_{i=1}^n\right)\right],
\end{align*}
where we use Jensen's inequality in the first inequality.
Due to the i.i.d. assumption of $\left\{u_i,\{\yb_{i,j}\}_{j=1}^{n_y}\right\}_{i=1}^n$, we have
\begin{align*}
    &\EE_{\cS}\left[ \exp \left(\left. t \left|z_1\right|^2\right|\left\{u_i,\{\yb_{i,j}\}_{j=1}^{n_y}\right\}_{i=1}^n\right)\right]=1+ \sum_{\ell=1}^{\infty} \frac{t^{\ell}\EE_{\cS} \left[ \left.  z_1^{2\ell}\right|\left\{u_i,\{\yb_{i,j}\}_{j=1}^{n_y}\right\}_{i=1}^n \right]}{\ell!}. 
\end{align*}
Since $z_1$ is sub-Gaussian with variance proxy $\sigma^2$, it follows that 
\begin{align*}
    &1+ \sum_{\ell=1}^{\infty} \frac{t^{\ell}\EE_{\cS} \left[ \left.  z_1^{2\ell}\right|\left\{u_i,\{\yb_{i,j}\}_{j=1}^{n_y}\right\}_{i=1}^n \right]}{\ell!} \\
    =& 1+ \sum_{\ell=1}^{\infty} \frac{t^{\ell}}{\ell!}\int_{0}^{\infty} \PP\left( \left. |z_1|\geq \tau^{\frac{1}{2\ell}}\right|\left\{u_i,\{\yb_{i,j}\}_{j=1}^{n_y}\right\}_{i=1}^n \right)d\tau\\
    \leq & 1+ 2\sum_{\ell=1}^{\infty} \frac{t^{\ell}}{\ell!}\int_0^{\infty} \exp\left( -\frac{\tau^{1/\ell}}{2\sigma^2}\right)d\tau\\
    =& 1+ \sum_{\ell=1}^{\infty} \frac{2\ell(2t\sigma^2)^{\ell}}{\ell!} \Gamma_G(\ell)\\
    =& 1+2\sum_{\ell=1}^{\infty} (2t\sigma^2)^{\ell},
\end{align*}
where $\Gamma_G$ denotes the Gamma function.  Setting $t=(4\sigma^2)^{-1}$, we have
\begin{align}
    \EE_{\cS} \left[\left. \max_k \left|z_k\right|^2\right|\left\{u_i,\{\yb_{i,j}\}_{j=1}^{n_y}\right\}_{i=1}^n \right] \leq& 4\sigma^2 \log \cN(\theta,\cG_{\rm NN},\|\cdot\|_{\infty,\infty}) + 4\sigma^2 \log 3 \nonumber\\
    \leq & 4\sigma^2 \log \cN(\theta,\cG_{\rm NN},\|\cdot\|_{\infty,\infty}) + 6\sigma^2.
    \label{eq.gene.T1.2.4}
\end{align}
Substituting (\ref{eq.gene.T1.2.4}) and (\ref{eq.gene.T1.2.3}) into (\ref{eq.gene.T1.2.2}) gives rise to
\begin{align}
    &\EE_{\cS}\left[\frac{1}{n}\sum_{i=1}^n \frac{1}{n_y}\sum_{j=1}^{n_y}\widehat{G}(\ub_i)(\yb_{i,j})\xi_{i,j}\right] \nonumber\\
	\leq & 2\sigma\left(\sqrt{\EE_{\cS}\left[\|\widehat{G} - G\|_n^2\right] } +\theta \right) \sqrt{\frac{4 \log \cN(\theta,\cG_{\rm NN},\|\cdot\|_{\infty,\infty}) + 6}{nn_y} } + \sigma\theta.
 \label{eq.gene.T1.2}
\end{align}
    
\end{proof}

\subsection{Proof of Lemma \ref{lem.gene.T2.err}}
\label{proof.lem.gene.T2.err}

\begin{proof}[Proof of Lemma \ref{lem.gene.T2.err}]

Denote $\widehat{g}(\ub)(\yb) = \left(\widehat{G}(\ub)(\yb)-G(u)(\yb)\right)^2$. Due to the clipping by $\beta_V$, $\|\widehat{g}\|_{\infty,\infty}\leq 4\beta_V^2$. Then 
\begin{align}
	{\rm T_2}=&\EE_{\cS}\left\{\EE_{u\sim \rho_u} \EE_{ \{\yb_j\}_{j = 1}^{n_y}\sim \rho_y}\left[\frac{1}{n_y}\sum_{j=1}^{n_y}\widehat{g}(\ub)(\yb_j)\right]-\frac{2}{n}\sum_{i=1}^n\frac{1}{n_y}\sum_{j=1}^{n_y}\widehat{g}(\ub_i)(\yb_{i,j})\right\} \nonumber\\
	=& 2\EE_{\cS}\left\{\frac{1}{2}\EE_{u\sim \rho_u} \EE_{ \{\yb_j\}_{j = 1}^{n_y} \sim \rho_y}\left[\frac{1}{n_y}\sum_{j=1}^{n_y}\widehat{g}(\ub)(\yb_j)\right]-\frac{1}{n}\sum_{i=1}^n\frac{1}{n_y}\sum_{j=1}^{n_y}\widehat{g}(\ub_i)(\yb_{i,j})\right\} \nonumber\\
	=& 2\EE_{\cS}\Bigg\{\EE_{u\sim \rho_u} \EE_{ \{\yb_j\}_{j = 1}^{n_y} \sim \rho_y }\left[\frac{1}{n_y}\sum_{j=1}^{n_y}\widehat{g}(\ub)(\yb_j)\right]-\frac{1}{n}\sum_{i=1}^n\frac{1}{n_y}\sum_{j=1}^{n_y}\widehat{g}(\ub_i)(\yb_{i,j}) \nonumber\\
	& - \frac{1}{2}\EE_{u\sim \rho_u} \EE_{\{\yb_j\}_{j = 1}^{n_y} \sim \rho_y}\left[\frac{1}{n_y}\sum_{j=1}^{n_y}\widehat{g}(\ub)(\yb_j)\right]\Bigg\}.
	\label{eq.gene.T2.1}
\end{align}
A lower bound of $\EE_{u\sim \rho_u} \EE_{  \{\yb_j\}_{j = 1}^{n_y} \sim \rho_y}\left[\frac{1}{n_y}\sum_{j=1}^{n_y}\widehat{g}(\ub)(\yb_j)\right]$ is given as
\begin{align}
	\EE_{u\sim \rho_u} \EE_{\{\yb_j\}_{j = 1}^{n_y} \sim \rho_y}\left[\frac{1}{n_y}\sum_{j=1}^{n_y}\hat{g}(\ub)(\yb_j)\right]=&\EE_{u\sim \rho_u} \EE_{\{\yb_j\}_{j = 1}^{n_y} \sim \rho_y }\left[\frac{1}{n_y}\sum_{j=1}^{n_y} \frac{4\beta_V^2}{4\beta_V^2}\widehat{g}(\ub)(\yb_j)\right] \nonumber\\
	\geq & \EE_{u\sim \rho_u} \EE_{\{\yb_j\}_{j = 1}^{n_y} \sim \rho_y }\left[\frac{1}{n_y}\sum_{j=1}^{n_y} \frac{1}{4\beta_V^2}\widehat{g}^2(\ub)(\yb_j)\right].
	\label{eq.gene.T2.lower}
\end{align}
Substituting (\ref{eq.gene.T2.lower}) into (\ref{eq.gene.T2.1}) implies
\begin{align}
	{\rm T_2} \leq &2\EE_{\cS}\Bigg\{\EE_{u\sim \rho_u}\EE_{ \{\yb_j\}_{j = 1}^{n_y} \sim \rho_y } \left[\frac{1}{n_y}\sum_{j=1}^{n_y}\widehat{g}(\ub)(\yb_j)\right]-\frac{1}{n}\sum_{i=1}^n\frac{1}{n_y}\sum_{j=1}^{n_y} \widehat{g}(\ub_i)(\yb_{i,j}) \nonumber\\
	& - \frac{1}{8\beta_V^2}\EE_{u\sim \rho_u} \EE_{ \{\yb_j\}_{j = 1}^{n_y} \sim \rho_y } \left[\frac{1}{n_y}\sum_{j=1}^{n_y} \widehat{g}^2(\ub)(\yb_j)\right]\Bigg\}.
	\label{eq.gene.T2.2}
\end{align}
Denote $\cS'=\left\{u_i',\{\yb_{i,j}'\}_{j=1}^{n_y}\right\}_{i=1}^n$ as an independent copy of $\cS$. 
Define the set
\begin{align}
	\cR= \{ g(\ub)(\yb) = (G_{\rm NN}(\ub)(\yb) - G(u)(\yb))^2 \text{ for } G_{\rm NN}\in \cG_{\rm NN}, u\in U, \yb\in \Omega_V \}.
 \label{def_difference_R}
\end{align}
We have
\begin{align}
	{\rm T_2}\leq  &2\EE_{\cS}\Bigg\{\sup_{g\in\cR} \Bigg(\EE_{\cS'} \left[\frac{1}{n}\sum_{i=1}^n\frac{1}{n_y}\sum_{j=1}^{n_y} g(\ub_i')(\yb_{i,j}')\right]-\frac{1}{n}\sum_{i=1}^n\frac{1}{n_y}\sum_{j=1}^{n_y} g(\ub_i)(\yb_{i,j}) \nonumber  \\
	& - \frac{1}{8\beta_V^2}\EE_{\cS'} \left[\frac{1}{n}\sum_{i=1}^n\frac{1}{n_y}\sum_{j=1}^{n_y} g^2(\ub_i')(\yb_{i,j}')\right]\Bigg)\Bigg\} \nonumber \\
	=& 2\EE_{\cS}\Bigg\{\sup_{g\in\cR} \Bigg(\EE_{\cS'} \left[\frac{1}{n}\sum_{i=1}^n\frac{1}{n_y}\sum_{j=1}^{n_y}\left(g(\ub_i')(\yb_{i,j}') - g(\ub_i)(\yb_{i,j})\right)\right] \nonumber \\
	& - \frac{1}{16\beta_V^2}\EE_{\cS,\cS'} \left[\frac{1}{n}\sum_{i=1}^n\frac{1}{n_y}\sum_{j=1}^{n_y} \left(g^2(\ub_i')(\yb_{i,j}') + g^2(\ub_i)(\yb_{i,j})\right)\right]\Bigg)\Bigg\} \nonumber \\
    \leq& {2\EE_{\cS,\cS'}\Bigg\{\sup_{g\in\cR} \Bigg(\frac{1}{n}\sum_{i=1}^n\Bigg(\frac{1}{n_y}\sum_{j=1}^{n_y} (g(\ub_i')(\yb_{i,j}')-g(\ub_i)(\yb_{i,j})} )\nonumber \\
	& {- \frac{1}{16\beta_V^2}\EE_{\cS,\cS'} \left[ \frac{1}{n_y}\sum_{j=1}^{n_y}\left(g^2(\ub_i')(\yb_{i,j}')+g^2(\ub_i)(\yb_{i,j})\right)\right]\Bigg)\Bigg) \Bigg\}.}
	\label{eq.gene.T2.3}
\end{align}
By Lemma \ref{lemma_cover_of_R}, let $\cR^*=\{g_k^*\}_{k=1}^{\cN(\theta,\cR,\|\cdot\|_{\infty,\infty})}$ be a $\theta$-cover of $\cR$. Then for any $g\in \cR$, there exists $g^*\in \cR^*$ satisfying $\|g-g^*\|_{\infty,\infty}\leq \theta$. We will derive an upper bound of (\ref{eq.gene.T2.3}) using $g^*$.

For the first term in (\ref{eq.gene.T2.3}), we have
\begin{align}
	&g(\ub_i')(\yb_{i,j}') - g(\ub_i)(\yb_{i,j}) \nonumber \\
	=& \left(g(\ub_i')(\yb_{i,j}') - g^*(\ub_i')(\yb_{i,j}')\right) + \left(g^*(\ub_i')(\yb_{i,j}') - g^*(\ub_i)(\yb_{i,j})\right) + \left(g^*(\ub_i)(\yb_{i,j}) - g(\ub_i)(\yb_{i,j})\right) \nonumber\\
	\leq & (g^*(\ub_i')(\yb_{i,j}') - g^*(\ub_i)(\yb_{i,j}))+ 2\theta.
	\label{eq.gene.T2.3.1}
\end{align}
For the second term in (\ref{eq.gene.T2.3}), we have
\begin{align}
	& g^2(\ub_i')(\yb_{i,j}') + g^2(\ub_i)(\yb_{i,j}) \nonumber\\
	= &\left(g^2(\ub_i')(\yb_{i,j}') - (g^*)^2(\ub_i')(\yb_{i,j}')\right)+ \left(g^2(\ub_i)(\yb_{i,j})-(g^*)^2(\ub_i)(\yb_{i,j})\right) \nonumber\\
	& + \left((g^*)^2(\ub_i)(\yb_{i,j})+(g^*)^2(\ub_i')(\yb_{i,j}')\right) \nonumber\\
	\geq& (g^*)^2(\ub_i)(\yb_{i,j})+(g^*)^2(\ub_i')(\yb_{i,j}')-|g(\ub_i')(\yb_{i,j}')-g^*(\ub_i')(\yb_{i,j}')| |g(\ub_i')(\yb_{i,j}') + g^*(\ub_i')(\yb_{i,j}')| \nonumber\\
	& -|g(\ub_i)(\yb_{i,j}) - g^*(\ub_i)(\yb_{i,j})||g(\ub_i)(\yb_{i,j}) + g^*(\ub_i)(\yb_{i,j})| \nonumber\\
	\geq & (g^*)^2(\ub_i)(\yb_{i,j})+(g^*)^2(\ub_i')(\yb_{i,j}') - 16\beta_V^2\theta.
	\label{eq.gene.T2.3.2}
\end{align}

Substituting (\ref{eq.gene.T2.3.1}) and (\ref{eq.gene.T2.3.2}) into (\ref{eq.gene.T2.3}) gives rise to
{\color{black}\begin{align}
	{\rm T_2} \leq & 2\EE_{\cS,\cS'}\Bigg\{\sup_{g\in\cR} \Bigg(\frac{1}{n}\sum_{i=1}^n\Bigg(\frac{1}{n_y}\sum_{j=1}^{n_y} (g(\ub_i')(\yb_{i,j}')-g(\ub_i)(\yb_{i,j})) \nonumber\\
	& - \frac{1}{16\beta_V^2}\EE_{\cS,\cS'} \left[\frac{1}{n_y} \left((g^*)^2(\ub_i)(\yb_{i,j})+(g^*)^2(\ub_i')(\yb_{i,j}')\right)\right]\Bigg)\Bigg)\Bigg] + 6\theta \nonumber \\
	=& 2\EE_{\cS,\cS'}\Bigg[\max_{k} \Bigg(\frac{1}{n}\sum_{i=1}^n\Bigg(\frac{1}{n_y}\sum_{j=1}^{n_y} (g_k^*(\ub_i')(\yb_{i,j}') - g_k^*(\ub_i)(\yb_{i,j}) )\nonumber \\
	& - \frac{1}{16\beta_V^2}\EE_{\cS,\cS'} \left[\frac{1}{n_y}\sum_{j=1}^{n_y} \left( (g_k^*)^2(\ub_i)(\yb_{i,j})+(g_k^*)^2(\ub_i')(\yb_{i,j}') \right)\right]\Bigg)\Bigg)\Bigg] + 6\theta
 \label{eq.gene.T2.4}
\end{align}}

{\color{black}
Denote $r_k\left(\ub_i',\{\yb_{i,j}'\}_{j=1}^{n_y},\ub_i,\{\yb_{i,j}\}_{j=1}^{n_y}\right)= \frac{1}{n_y}\sum_{j=1}^{n_y}(g_k^*(\ub_i')(\yb_{i,j}') - g_k^*(\ub_i)(\yb_{i,j}))$. We have
\begin{align*}
    &\EE_{\cS,\cS'}\left[r_k\left(\ub_i',\{\yb_{i,j}'\}_{j=1}^{n_y},\ub_i,\{\yb_{i,j}\}_{j=1}^{n_y}\right)\right]=0, 
\end{align*}
and thus
\begin{align}
    &\Var\left(r_k\left(\ub_i',\{\yb_{i,j}'\}_{j=1}^{n_y},\ub_i,\{\yb_{i,j}\}_{j=1}^{n_y}\right)\right) \nonumber\\
    =&\EE_{\cS,\cS'}\left[r_k^2\left(\ub_i',\{\yb_{i,j}'\}_{j=1}^{n_y},\ub_i,\{\yb_{i,j}\}_{j=1}^{n_y}\right)\right] \nonumber\\
    =&\EE_{\cS,\cS'}\left[\frac{1}{n_y^2}\left(\sum_{j=1}^{n_y} \big(g_k^*(\ub_i')(\yb_{i,j}') - g_k^*(\ub_i)(\yb_{i,j})\big)\right)^2\right] \nonumber\\
    =&\EE_{\cS,\cS'}\left[\frac{1}{n_y^2}\sum_{j=1}^{n_y}\sum_{\ell=1}^{n_y} \big(g_k^*(\ub_i')(\yb_{i,j}') - g_k^*(\ub_i)(\yb_{i,j})\big)\big(g_k^*(\ub_i')(\yb_{i,\ell}') - g_k^*(\ub_i)(\yb_{i,\ell})\big)\right] \nonumber\\
    =&\EE_{\cS,\cS'}\left[\frac{1}{n_y^2}\sum_{j=1}^{n_y} \big(g_k^*(\ub_i')(\yb_{i,j}') - g_k^*(\ub_i)(\yb_{i,j})\big)^2\right] \nonumber\\
    &+\EE_{\cS,\cS'}\left[\frac{1}{n_y^2}\sum_{j=1}^{n_y}\sum_{\ell\neq j} \big(g_k^*(\ub_i')(\yb_{i,j}') - g_k^*(\ub_i)(\yb_{i,j})\big)\big(g_k^*(\ub_i')(\yb_{i,\ell}') - g_k^*(\ub_i)(\yb_{i,\ell})\big)\right]
    \label{eq.S.var.0}
\end{align}
Note that for $\ell\neq j$, we have
\begin{align}
    &\EE_{\cS,\cS'}\left[ \big(g_k^*(\ub_i')(\yb_{i,j}') - g_k^*(\ub_i)(\yb_{i,j})\big)\big(g_k^*(\ub_i')(\yb_{i,\ell}') - g_k^*(\ub_i)(\yb_{i,\ell})\big)\right] \nonumber\\
    =& \EE_{u_i\sim \rho_u, u_i'\sim \rho_u}\Big[\EE_{y_{i,j},y_{i,\ell},y'_{i,j},y'_{y,\ell}\sim \rho_y} \Big[\big(g_k^*(\ub_i')(\yb_{i,j}') - g_k^*(\ub_i)(\yb_{i,j})\big)\big(g_k^*(\ub_i')(\yb_{i,\ell}') - g_k^*(\ub_i)(\yb_{i,\ell})\big) \Big| u_i,u'_i\Big]\Big] \nonumber\\
    =&\EE_{u_i\sim \rho_u, u_i'\sim \rho_u}\Big[\Big(\EE_{y_{i,j},y'_{i,j}\sim \rho_y} \Big[\big(g_k^*(\ub_i')(\yb_{i,j}') - g_k^*(\ub_i)(\yb_{i,j})\big)\Big]\Big)^2\Big] \nonumber\\
    \leq& \EE_{u_i\sim \rho_u, u_i'\sim \rho_u}\Big[\EE_{y_{i,j},y'_{i,j}\sim \rho_y} \Big[\big(g_k^*(\ub_i')(\yb_{i,j}') - g_k^*(\ub_i)(\yb_{i,j})\big)^2\Big]\Big] \nonumber\\
    =&\EE_{\cS,\cS'}\Big[\big(g_k^*(\ub_i')(\yb_{i,j}') - g_k^*(\ub_i)(\yb_{i,j})\big)^2\Big],
    \label{eq.S.var.mix}
\end{align}
where the inequality uses Jensen's inequality.
Substituting (\ref{eq.S.var.mix}) into (\ref{eq.S.var.0}) gives rise to
\begin{align}
    &\Var\left(r_k\left(\ub_i',\{\yb_{i,j}'\}_{j=1}^{n_y},\ub_i,\{\yb_{i,j}\}_{j=1}^{n_y}\right)\right) \nonumber\\
    \leq & \EE_{\cS,\cS'}\left[\frac{1}{n_y^2}\sum_{j=1}^{n_y} \big(g_k^*(\ub_i')(\yb_{i,j}') - g_k^*(\ub_i)(\yb_{i,j})\big)^2\right]+\EE_{\cS,\cS'}\left[\frac{1}{n_y^2}\sum_{j=1}^{n_y}\sum_{\ell\neq j} \big(g_k^*(\ub_i')(\yb_{i,j}') - g_k^*(\ub_i)(\yb_{i,j})\big)^2\right]
    \nonumber\\
    =& \EE_{\cS,\cS'}\left[\frac{1}{n_y^2}\sum_{j=1}^{n_y}\sum_{\ell=1}^{n_y} \big(g_k^*(\ub_i')(\yb_{i,j}') - g_k^*(\ub_i)(\yb_{i,j})\big)^2\right] \nonumber\\
    =&\EE_{\cS,\cS'}\left[\frac{1}{n_y}\sum_{j=1}^{n_y} \big(g_k^*(\ub_i')(\yb_{i,j}') - g_k^*(\ub_i)(\yb_{i,j})\big)^2\right] \nonumber\\
    \leq & 2\EE_{\cS,\cS'}\left[\frac{1}{n_y}\sum_{j=1}^{n_y} \big((g_k^*)^2(\ub_i')(\yb_{i,j}')+(g_k^*)^2(\ub_i)(\yb_{i,j})\big)^2\right].
    \label{eq.S.var.1}
\end{align}
}
Next we define and estimate $\tilde{T}_2$,
{\color{black}
\begin{align*}
\widetilde{\rm T}_2 \leq &2\EE_{\cS,\cS'}\Bigg[\max_{k} \Bigg(\frac{1}{n}\sum_{i=1}^n\Bigg( r_k\left(\ub_i',\{\yb_{i,j}'\}_{j=1}^{n_y},\ub_i,\{\yb_{i,j}\}_{j=1}^{n_y}\right) 
\\
&\hspace{2cm}- \frac{1}{16\beta_V^2}\Var\left[r_k\left(\ub_i',\{\yb_{i,j}'\}_{j=1}^{n_y},\ub_i,\{\yb_{i,j}\}_{j=1}^{n_y}\right)\right]\Bigg)\Bigg)\Bigg].
\end{align*}
According to (\ref{eq.gene.T2.4}) and (\ref{eq.S.var.1}),
$$
{\rm T_2}\leq \widetilde{\rm T}_2+6\theta.
$$
}
We estimate $\tilde{T}_2$ using the moment generating function of $r_k$. Note that $\|r_k\|_{\infty,\infty}\leq 4\beta_V^2$. For $0<t<3/4\beta_V^2$, we have
\begin{align}
	&\EE_{\cS,\cS'}\left[\exp(tr_k\left(\ub_i',\{\yb_{i,j}'\}_{j=1}^{n_y},\ub_i,\{\yb_{i,j}\}_{j=1}^{n_y}\right))\right] \nonumber\\
	=&\EE_{\cS,\cS'}\left[ 1+tr_k\left(\ub_i',\{\yb_{i,j}'\}_{j=1}^{n_y},\ub_i,\{\yb_{i,j}\}_{j=1}^{n_y}\right)+ \sum_{\ell=2}^{\infty} \frac{t^{\ell}r_k^{\ell}(\ub_i',\{\yb_{i,j}'\}_{j=1}^{n_y},\ub_i,\{\yb_{i,j}\}_{j=1}^{n_y})}{\ell!} \right] \nonumber\\
	\leq & \EE_{\cS,\cS'}\left[ 1+tr_k\left(\ub_i',\{\yb_{i,j}'\}_{j=1}^{n_y},\ub_i,\{\yb_{i,j}\}_{j=1}^{n_y}\right)+ \sum_{\ell=2}^{\infty} \frac{(4\beta_V^2)^{\ell-2} t^{\ell}r_k^2\left(\ub_i',\{\yb_{i,j}'\}_{j=1}^{n_y},\ub_i,\{\yb_{i,j}\}_{j=1}^{n_y}\right)}{2\times 3^{\ell-2}} \right] \nonumber\\
	=& \EE_{\cS,\cS'}\left[ 1+tr_k\left(\ub_i',\{\yb_{i,j}'\}_{j=1}^{n_y},\ub_i,\{\yb_{i,j}\}_{j=1}^{n_y}\right)+ \frac{\ell^2r_k^2\left(\ub_i',\{\yb_{i,j}'\}_{j=1}^{n_y},\ub_i,\{\yb_{i,j}\}_{j=1}^{n_y}\right)}{2}\sum_{\ell=2}^{\infty} \frac{(4\beta_V^2)^{\ell-2} t^{\ell-2}}{3^{\ell-2}} \right] \nonumber\\
	=& \EE_{\cS,\cS'} \left[ 1+tr_k\left(\ub_i',\{\yb_{i,j}'\}_{j=1}^{n_y},\ub_i,\{\yb_{i,j}\}_{j=1}^{n_y}\right)+ \frac{\ell^2r_k^2\left(\ub_i',\{\yb_{i,j}'\}_{j=1}^{n_y},\ub_i,\{\yb_{i,j}\}_{j=1}^{n_y}\right)}{2}\frac{1}{1-4\beta_V^2t/3} \right] \nonumber\\
	=&1+t^2\Var[r_k\left(\ub_i',\{\yb_{i,j}'\}_{j=1}^{n_y},\ub_i,\{\yb_{i,j}\}_{j=1}^{n_y}\right)]\frac{1}{2-8\beta_V^2t/3} \nonumber\\
	\leq & \exp\left( \Var[r_k\left(\ub_i',\{\yb_{i,j}'\}_{j=1}^{n_y},\ub_i,\{\yb_{i,j}\}_{j=1}^{n_y}\right)]\frac{3t^2}{6-8\beta_V^2t} \right),
	\label{eq.gene.T2.r.moment}
\end{align}
where the last inequality comes from the relation $1+x\leq \exp(x)$ for $x\geq 0$.

{\color{black}
For $0<t/n<3/4\beta_V^2$, we have
\begin{align}
	&\exp\left(\frac{t \widetilde{\rm T}_2}{2}\right) \nonumber\\
	=& \exp\Bigg( t\EE_{\cS,\cS'}\Bigg[\max_{k} \Bigg(\frac{1}{n}\sum_{i=1}^n\Bigg( r_k\left(\ub_i',\{\yb_{i,j}'\}_{j=1}^{n_y},\ub_i,\{\yb_{i,j}\}_{j=1}^{n_y}\right) \nonumber\\
    &\hspace{5cm}- \frac{1}{16\beta_V^2}\Var\left[r_k\left(\ub_i',\{\yb_{i,j}'\}_{j=1}^{n_y},\ub_i,\{\yb_{i,j}\}_{j=1}^{n_y}\right)\right]\Bigg)\Bigg)\Bigg]\Bigg) \nonumber\\
	\leq& \EE_{\cS,\cS'}\Bigg[\exp\Bigg( t\max_{k} \Bigg(\frac{1}{n}\sum_{i=1}^n\Bigg( r_k\left(\ub_i',\{\yb_{i,j}'\}_{j=1}^{n_y},\ub_i,\{\yb_{i,j}\}_{j=1}^{n_y}\right) \nonumber\\
    &\hspace{5cm}- \frac{1}{16\beta_V^2}\Var\left[r_k\left(\ub_i',\{\yb_{i,j}'\}_{j=1}^{n_y},\ub_i,\{\yb_{i,j}\}_{j=1}^{n_y}\right)\right]\Bigg)\Bigg)\Bigg) \Bigg] \nonumber\\
	\leq & \EE_{\cS,\cS'}\Bigg[\sum_k\exp\Bigg( \sum_{i=1}^n\Bigg( \frac{t}{n}r_k\left(\ub_i',\{\yb_{i,j}'\}_{j=1}^{n_y},\ub_i,\{\yb_{i,j}\}_{j=1}^{n_y}\right) \nonumber\\
    &\hspace{5cm}- \frac{1}{n}\frac{t}{16\beta_V^2}\Var\left[r_k\left(\ub_i',\{\yb_{i,j}'\}_{j=1}^{n_y},\ub_i,\{\yb_{i,j}\}_{j=1}^{n_y}\right)\right]\Bigg)\Bigg) \Bigg] \nonumber\\
    %% add one line
    \leq & \sum_k \EE_{\cS,\cS'}\prod_{i=1}^n\exp \Bigg( \frac{t}{n}r_k\left(\ub_i',\{\yb_{i,j}'\}_{j=1}^{n_y},\ub_i,\{\yb_{i,j}\}_{j=1}^{n_y}\right) \nonumber\\
    &\hspace{5cm}- \frac{1}{n}\frac{t}{16\beta_V^2}\Var\left[r_k\left(\ub_i',\{\yb_{i,j}'\}_{j=1}^{n_y},\ub_i,\{\yb_{i,j}\}_{j=1}^{n_y}\right)\right]\Bigg)  \nonumber\\
	\leq & \sum_k\exp\Bigg( \sum_{i=1}^n\Bigg( \Var[r_k\left(\ub_i',\{\yb_{i,j}'\}_{j=1}^{n_y},\ub_i,\{\yb_{i,j}\}_{j=1}^{n_y}\right)]\frac{3(t/n)^2}{6-8\beta_V^2t/n} \nonumber\\
    &\hspace{5cm}- \frac{1}{n}\frac{t}{16\beta_V^2}\Var\left[r_k\left(\ub_i',\{\yb_{i,j}'\}_{j=1}^{n_y},\ub_i,\{\yb_{i,j}\}_{j=1}^{n_y}\right)\right]\Bigg)\Bigg) \nonumber\\
	=& \sum_k\exp\left( \sum_{i=1}^n \frac{t}{n}\Var[r_k\left(\ub_i',\{\yb_{i,j}'\}_{j=1}^{n_y},\ub_i,\{\yb_{i,j}\}_{j=1}^{n_y}\right)]\left(\frac{3t/n}{6-8\beta_V^2t/n} - \frac{1}{16\beta_V^2}\right)\right)
	\label{eq.gene.T2.tilde.1}
\end{align}
where the first inequality follows from Jensen’s inequality and the third inequality uses (\ref{eq.gene.T2.r.moment}) by replacing $t$ by $t/n$. Setting 
$$
\frac{3t/n}{6-8\beta_V^2t/n} - \frac{1}{16\beta_V^2}=0,
$$
we have $t=\frac{3n}{28\beta_V^2}$ and $\frac{t}{n}<\frac{3}{4\beta_V^2}$. Substituting the choice of $t$ into (\ref{eq.gene.T2.tilde.1}) gives rise to
\begin{align*}
	\frac{t\widetilde{\rm T}_2}{2}\leq \log \sum_k \exp(0).
\end{align*}
Thus 
\begin{align*}
	\widetilde{\rm T}_2\leq \frac{2}{t}\log \cN(\theta,\cR,\|\cdot\|_{\infty,\infty})=\frac{56\beta_V^2}{3nn_y}\log \cN(\theta,\cR,\|\cdot\|_{\infty,\infty})
\end{align*}
and 
\begin{align}
	{\rm T}_2\leq \frac{56\beta_V^2}{3n}\log \cN(\theta,\cR,\|\cdot\|_{\infty,\infty})+6\theta\leq \frac{19\beta_V^2}{n}\log \cN(\theta,\cR,\|\cdot\|_{\infty,\infty})+6\theta.
	\label{eq.gene.T2.5}
\end{align}
}
The following lemma (see a proof in Section \ref{proof.lemma_cover_of_R}) gives a relation between 
the covering number of $\cR$ and $\cG_{\rm NN}$:
\begin{lemma}\label{lemma_cover_of_R}
    Let $\mathcal{G}^*$ be a $\theta$ cover with the covering number $\cN(\theta,\cG_{\rm NN}, \|\cdot\|_{\infty,\infty})$. There exists a finite $\theta$ cover $\mathcal{R}^*$ of $\mathcal{R}$, and the covering number is bounded by,
    \begin{align*}
        \cN(\theta,\cR,\|\cdot\|_{\infty,\infty})\leq \cN\left(\frac{\theta}{4\beta_V},\cG_{\rm NN},\|\cdot\|_{\infty,\infty}\right).
    \end{align*}
\end{lemma}

{\color{black}
By Lemma \ref{lemma_cover_of_R}, we have
\begin{align}
	{\rm T}_2\leq \frac{19\beta_V^2}{n}\log \cN\left(\frac{\theta}{4\beta_V},\cG_{\rm NN},\|\cdot\|_{\infty,\infty}\right)+6\theta.
	\label{eq.gene.T2.err}
\end{align}
}
\end{proof}

\subsection{Proof of Lemma \ref{lem.covering}}
\label{proof.covering}
\begin{proof}[Proof of Lemma \ref{lem.covering}]
    For $h>0$ and each $k$, let $\widetilde{a}_k'\in \cF_2$ be some network so that each nonzero parameter of $\widetilde{a}_k'$ is at most different from the corresponding one in $\widetilde{a}_k$ by $h$. Similarly, let $\widetilde{q}_k'\in \cF_1$ be some network so that each nonzero parameter of $\widetilde{q}_k'$ is at most different from the corresponding one in $\widetilde{q}_k$ by $h$.

    According to \citet[Proof of Lemma 5.3]{chen2022nonparametric}, we have
    \begin{align*}
        \|\widetilde{a}_k'-\widetilde{a}_k\|_{\infty}\leq hL_2(p_2\beta_U+2)(\kappa_2p_2)^{L_2-1}, \quad \|\widetilde{q}_k'-\widetilde{q}_k\|_{\infty}\leq hL_1(p_1\gamma_2+2)(\kappa_1p_1)^{L_1-1}.
    \end{align*}
    We deduce 
    \begin{align*}
    &\left\|\CL_c\left(\sum_{k=1}^{N} \widetilde{a}'_k(\ub )\widetilde{q}'_k(\yb)\right)-\CL_c\left(\sum_{k=1}^{N} \widetilde{a}_k(\ub )\widetilde{q}_k(\yb)\right)\right\|_{\infty,\infty}\\
        \leq &\left\|\sum_{k=1}^{N} \widetilde{a}'_k(\ub )\widetilde{q}'_k(\yb)-\sum_{k=1}^{N} \widetilde{a}_k(\ub )\widetilde{q}_k(\yb)\right\|_{\infty,\infty} \\
        \leq &\sum_{k=1}^N\|\widetilde{a}'_k(\ub )\widetilde{q}'_k(\yb)-\widetilde{a}_k(\ub )\widetilde{q}_k(\yb)\|_{\infty,\infty}\\
        \leq& \sum_{k=1}^N\left(\|\widetilde{a}'_k(\ub )\widetilde{q}'_k(\yb)-\widetilde{a}'_k(\ub )\widetilde{q}_k(\yb)\|_{\infty,\infty} +\|\widetilde{a}'_k(\ub )\widetilde{q}_k(\yb)-\widetilde{a}_k(\ub )\widetilde{q}_k(\yb)\|_{\infty,\infty} \right)\\
        \leq& \sum_{k=1}^N\left(\|\widetilde{a}'_k\|_{\infty}\|\widetilde{q}'_k(\yb)-\widetilde{q}_k(\yb)\|_{\infty} +\|\widetilde{q}_k\|_{\infty}\|\widetilde{a}'_k(\ub )-\widetilde{a}_k(\ub)\|_{\infty} \right)\\
        \leq& \sum_{k=1}^N\left(R_1hL_1(p_1\gamma_2+2)(\kappa_1p_1)^{L_1-1}+R_2hL_2(p_2\beta_U+2)(\kappa_2p_2)^{L_2-1} \right)\\
        =&hN\left(R_1L_1(p_1\gamma_2+2)(\kappa_1p_1)^{L_1-1}+R_2L_2(p_2\beta_U+2)(\kappa_2p_2)^{L_2-1} \right).
    \end{align*}
    Set $h$ so that $hN\left(R_1L_1(p_1\gamma_2+2)(\kappa_1p_1)^{L_1-1}+R_2L_2(p_2\beta_U+2)(\kappa_2p_2)^{L_2-1} \right)=\theta$ gives 
    $$
    h=\frac{\theta}{H}  \mbox{ with }
H=N\left(R_1L_1(p_1\gamma_2+2)(\kappa_1p_1)^{L_1-1}+R_2L_2(p_2\beta_U+2)(\kappa_2p_2)^{L_2-1} \right). 
$$

    We uniformly discretize the parameters of $\cF_1$ and $\cF_2$ by $2\kappa_1/h$ and $2\kappa_2/h$. The collection of all networks corresponding to those grid parameters forms a $\theta$-cover of $\cG$. The covering number is bounded by
    \begin{align*}
        \cN(\theta,\cG_{NN},\|\cdot\|_{\infty,\infty})\leq& \binom{L_1p_1^2}{K_1}\left(\frac{2\kappa_1}{h}\right)^{NK_1}\cdot \binom{L_2p_2^2}{K_2}\left(\frac{2\kappa_2}{h}\right)^{NK_2}\\
        \leq & (L_1p_1^2)^{NK_1}\left(\frac{2\kappa_1}{h}\right)^{NK_1}\cdot(L_2p_2^2)^{K_2}\left(\frac{2\kappa_2}{h}\right)^{K_2}\\
        \leq & \left(\frac{2L_1p_1^2\kappa_1H}{\theta}\right)^{NK_1}\left(\frac{2L_2p_2^2\kappa_2H}{\theta}\right)^{NK_2}.
    \end{align*}
\end{proof}

\subsection{Proof of Lemma \ref{lemma_cover_of_R}}
\label{proof.lemma_cover_of_R}
\begin{proof}[Proof of Lemma \ref{lemma_cover_of_R}]
    For any $g,\bar{g}\in \cR$, we have
\begin{align}
	g(\ub)(\yb)=(G_{\rm NN}(\ub)(\yb) - G(u)(\yb))^2, \quad \bar{g}(\ub)(\yb)=(\bar{G}_{\rm NN}(\ub)(\yb) - G(u)(\yb))^2
 \label{eq.gene.covering.relation}
\end{align}
for some $G_{\rm NN},\bar{G}_{\rm NN}\in \cG$. We have
\begin{align*}
	\|g-\bar{g}\|_{\infty,\infty}=&\sup_{u\in U}\sup_{\yb\in \Omega_V} |(G_{\rm NN}(\ub)(\yb)-G(u)(\yb))^2 - (\bar{G}_{\rm NN}(\ub)(\yb)-G(u)(\yb))^2|\\
	=& \sup_{u\in U}\sup_{\yb\in \Omega_V} | \left(G_{\rm NN}(\ub)(\yb)-\bar{G}_{\rm NN}(\ub)(\yb)\right)\left(G_{\rm NN}(\ub)(\yb)+\bar{G}_{\rm NN}(\ub)(\yb) - 2G(u)(\yb)\right)|\\
	\leq& \sup_{u\in U}\sup_{\yb\in \Omega_V} |G_{\rm NN}(\ub)(\yb)-\bar{G}_{\rm NN}(\ub)(\yb)||G_{\rm NN}(\ub)(\yb) + \bar{G}_{\rm NN}(\ub)(\yb) - 2G(u)(\yb)|\\
	\leq & 4\beta_V\|G_{\rm NN}-\bar{G}_{\rm NN}\|_{\infty,\infty}.
\end{align*}
We thus have
\begin{align*}
	\cN(\theta,\cR,\|\cdot\|_{\infty,\infty})\leq \cN\left(\frac{\theta}{4\beta_V},\cG_{NN},\|\cdot\|_{\infty,\infty}\right).
\end{align*}
\end{proof}

\end{document}